\documentclass[10pt, journal]{IEEEtran}

\usepackage{amsmath,amssymb,amsthm}
\usepackage{graphicx,here,comment}
\usepackage{algorithm,algorithmic}
\usepackage{xcolor}
\usepackage{url}

\newtheorem{theorem}{Theorem}
\newtheorem{corollary}[theorem]{Corollary}
\newtheorem{lemma}[theorem]{Lemma}
\newtheorem{proposition}[theorem]{Proposition}
\newtheorem{assumption}{Assumption}
\theoremstyle{definition}
\newtheorem{definition}{Definition}

\newcommand{\R}{\mathbb{R}}
\newcommand{\N}{\mathbb{N}}
\newcommand{\mF}{\mathcal{F}}
\newcommand{\mG}{\mathcal{G}}

\newcommand{\mA}{\mathcal{A}}
\newcommand{\mB}{\mathcal{B}}
\newcommand{\mfB}{\mathfrak{B}}

\newcommand{\mD}{\mathcal{D}}

\newcommand{\mM}{\mathcal{M}}
\newcommand{\mN}{\mathcal{N}}
\newcommand{\mX}{\mathcal{X}}
\newcommand{\mY}{\mathcal{Y}}

\newcommand{\fR}{\mathfrak{R}}

\newcommand{\Ep}{\mathbb{E}}
\renewcommand{\Pr}{\mathbb{P}}

\newcommand{\mE}{\mathcal{E}}

\renewcommand{\hat}{\widehat}
\renewcommand{\tilde}{\widetilde}

\newcommand{\argmin}{\operatornamewithlimits{argmin}}

\newcommand{\mone}{\textbf{1}}

\DeclareMathOperator{\Var}{Var}
\DeclareMathOperator{\Cov}{Cov}

\usepackage{color}

\usepackage{newtxtext,newtxmath}

\title{On Generalization Bounds for Deep Networks \\ based on Loss Surface Implicit Regularization}

\author{Masaaki Imaizumi and Johannes Schmidt-Hieber
\thanks{
Masaaki Imaizumi received support from JSPS KAKENHI (18K18114) and JST Presto (JPMJPR1852).
Johannes Schmidt-Hieber received support from the Dutch Science Foundation (NWO) via the Vidi grant VI.Vidi.192.021.

Masaaki Imaizumi is with Komaba Institute for Science, the University of Tokyo, Japan. 
(e-mail imaizumi@g.ecc.u-tokyo.ac.jp)

Johannes Schmidt-Hieber is with the University of Twente, Netherlands. 
(e-mail a.j.schmidt-hieber@utwente.nl)
}
}

\date{\today}

\allowdisplaybreaks

\providecommand{\keywords}[1]
{
  \small    
  \textbf{\textit{Keywords---}} #1
}

\begin{document}
\maketitle

\begin{abstract}
The classical statistical learning theory implies that fitting too many parameters leads to overfitting and poor performance. That modern deep neural networks generalize well despite a large number of parameters contradicts this finding and constitutes a major unsolved problem towards explaining the success of deep learning. While previous work focuses on the implicit regularization induced by stochastic gradient descent (SGD), we study here how the local geometry of the energy landscape around local minima affects the statistical properties of SGD with Gaussian gradient noise. We argue that under reasonable assumptions, the local geometry forces SGD to stay close to a low dimensional subspace and that this induces another form of implicit regularization and results in tighter bounds on the generalization error for deep neural networks. To derive generalization error bounds for neural networks, we first introduce a notion of stagnation sets around the local minima and impose a local essential convexity property of the population risk. Under these conditions, lower bounds for SGD to remain in these stagnation sets are derived. If stagnation occurs, we derive a bound on the generalization error of deep neural networks involving the spectral norms of the weight matrices but not the number of network parameters. Technically, our proofs are based on controlling the change of parameter values in the SGD iterates and local uniform convergence of the empirical loss functions based on the entropy of suitable neighborhoods around local minima. 
\end{abstract}


\section{Introduction}

We consider supervised learning using deep neural networks. Let $(X_1,Y_1),...,(X_n,Y_n)$ be $n$ i.i.d. observed pairs in $\mX \times \mY$ with input space $\mX$ and output space $\mY$. Write $f(\cdot;\mA): \mX \to \mY$ for a function generated by a neural network with $L$ layers and parameter matrices $\mA = (A_1,...,A_L).$ The total number of parameters is denoted by $D.$ We assume that the set of possible network parameters is constrained to lie in a known parameter space $\Theta \subseteq \R^D.$ Let $\ell: \mY \times \mY \to \R$ be the loss function and denote by $\hat{\mA}$ the output of the stochastic gradient descent (SGD) algorithm based on the empirical loss $R_n(\mA) := n^{-1} \sum_{i=1}^n \ell (Y_i, f(X_i;\mA)).$ The expected loss (generalization error) is then given by $R(\mA): =\Ep[\ell(Y,f(X;\mA))]$. Throughout this manuscript, we assume that the loss if sufficiently regular such that the gradient $\nabla R(\mA)$ exists for all $\mA$.

To achieve small generalization error, we are interested in bounds for the \textit{generalization gap}
\begin{align*}
    R(\hat{\mA}) - R_n(\hat{\mA})
\end{align*}
with $\hat{\mA}$ the SGD output. In this work, we derive a new bound for the generalization gap, by exploiting the tendency of SGD iterates to stay in neighborhoods of critical points in the loss surface. We then argue that for complex loss surfaces, this constraints the SGD iterates considerably and leads to implicit regularization. Throughout this article we refer to this type of implicit regularization as \textit{loss surface implicit regularization}.

\subsection{Background}
\label{sec.backg}

Deep learning is known to achieve outstanding performance in various complicated tasks, such as image recognition, text analysis, and reinforcement learning \cite{lecun2015deep}. For some tasks, super-human performance has been reported \cite{he2016deep,devlin2019bert, NEURIPS2020_1457c0d6}. Despite the impressive practical performance, there is still a gap regarding the theoretical understanding of deep learning. Obstacles in the theoretical foundation include the higher-order nonlinear structures due to the stacking of multiple layers and the excessive number of network parameters in state of the art networks. For some recent surveys, see \cite{bartlett2021deep,berner2021modern}.

To understand the use of fitting large numbers of parameters, requires to \textit{rethink generalization} \cite{zhou2019understanding}. Indeed, according to the standard statistical learning theory, large models overfit and thus increase the generalization error.
Based on the classical analysis (e.g., the textbook \cite{anthony2009neural}), the generalization gap for deep neural networks with $L$ layers and $D$ parameters is of the order
\begin{align*}
    R(\hat{\mA}) - R_n(\hat{\mA})  = \tilde{O}_\Pr(\sqrt{{LD}/{n}} )
\end{align*}
where $\tilde{O}_\Pr(\cdot)$ stands for the Big O notation ignoring logarithmic factors and the index $\Pr$ indicates that the rates are in probability. While bounds of this type can be used to prove optimal statistical convergence rates for sparsely connected neural networks \cite{MR4134774, pmlr-v89-imaizumi19a,imaizumi2022advantage,suzuki2018adaptivity}, the bound is clearly not sharp enough to explain the success of fully connected highly overparametrized models. Hence, a new theoretical framework is needed.

One possibility to derive sharper bounds on the generalization gap is to take the \textit{implicit regularization} into account. Learning algorithms, such as SGD, implicitly regularize and constraint the degrees of freedom. In several specific models, including linear regression, logistic regression and linear neural networks, it has been shown that gradient descent methods converges to interpolants with minimum norm constraints \cite{ji2018risk,gunasekar2018characterizing,gunasekar2017implicit,gunasekar2018implicit,nacson2019convergence,li2018algorithmic, MR4134779, 2020arXiv200607356J}.
Motivated by this fact, several articles investigate the generalization gap under the assumption that a norm on the network parameters is bounded by a threshold. For example, \cite{neyshabur2015norm} assumes that the network parameters lie in a norm-bounded subset $\tilde{\mB} := \left\{ \mA =(A_1,...,A_L) \mid \|A_\ell \|_F \leq b_\ell, \ell = 1,...,L \right\} \subset \Theta$ with $b_1,\dots,b_L$ given and $\|\cdot\|_F$ the Frobenius norm. It is shown that the generalization gap is then $R(\hat{\mA}) - R_n(\hat{\mA}) = \tilde{O}_\Pr\left( {2^L \Pi_{\ell=1}^L b_\ell}/{\sqrt{n}} \right).$ Interestingly, the bound only depends on the number of layers $L$, the number of training data $n$ and the radii $b_\ell,$ but not on the number of network parameters. This shows that norm control can avoid overfitting even in the case of overparametrization. \cite{bartlett2017spectrally} derives the generalization gap bound $R(\hat{\mA}) - R_n(\hat{\mA}) = \tilde{O}_\Pr( { \Pi_{\ell=1}^L s_\ell (\sum_{\ell=1}^L (\check{b}_\ell/s_\ell)^{2/3} )^{3/2}}/{\sqrt{n}} )$ for the set of network parameters $\check{\mB} := \{ \mA =(A_1,...,A_L) $ $\mid \|A_\ell \|_{2,1} \leq \check{b}_\ell, \|A_\ell \|_{s} \leq s_\ell, \ell = 1,...,L \},$ where $\|A\|_{2,1}$ denotes the sum of the Euclidean norms of the rows of the matrix $A$ and $\|\cdot\|_s$ is the spectral norm. This bound has been extended and improved in subsequent work \cite{neyshabur2017exploring,golowich2017size}.
Especially, \cite{bartlett2017spectrally} bounds the generalization gap by a product of the spectral norms of the parameter matrices.
Alternatively, \cite{hardt2016train} derived a bound on the generalization gap involving the parameter distance between the initial value and a global minimizer. Section \ref{sec:related_works} provides a more comprehensive overview of related work. 

The imposed constraints on the parameter norms might be violated in practice. \cite{nagarajan2019uniform} shows that during network training, the parameter matrices $\mA$ move far away from the origin and the initialization. Empirically, it is argued that the distance to the initialization increases polynomially with the number of training data. The claim is that only for simple models, implicit regularization favors small norms.

\subsection{Summary of Results}
We first define suitable neighborhoods around the local minima of the loss surface. 
Under a number of conditions, we then prove that SGD with Gaussian gradient noise enters such a neighborhood and will not escape it with positive probability. 
For this reason, we also call the neighborhoods stagnation sets. In a second step, we derive bounds for the complexity of these neighborhoods. Conditionally on the SGD iterates lying in one of these neighborhoods, we finally derive a generalization gap bound. Based on these results, we then argue that the loss surface itself constraints the SGD resulting in the loss surface implicit regularization.

To define a suitable notion of a \textit{stagnation set}, we consider a sequence of parameters $\{\mA_t\}_{t=1}^T$ generated by Gaussian SGD with an iteration index $t \in \N$ and a stopping time $T \in \N.$ The initial value is $\mA_0 \in \Theta$ and the iterates are defined via the update equation $\mA_{t+1} = \mA_t - \eta_t \nabla \hat{R}_n( \mA_t)$. Here, $\eta_t > 0$ is a given learning rate and $\hat{R}_n$ is a perturbed loss with Gaussian noise such that the updates resemble SGD (a formal definition is provided in \eqref{def:reguralized_SGD}). The output of the method is $\hat{\mA} = \mA_T$.
We define the notion of a stagnation set as follows:
\begin{definition}[Stagnation Set]
    For $p \in [0,1]$, $\mB \subset \Theta$ is a $p$-\textit{stagnation set}, if the following holds with some $\Bar{t} \geq 1$:
    \begin{align*}
        \Pr \left(\mA_t \in \mB, \forall t \in [\Bar{t}, T] \right) \geq p.
    \end{align*}
\end{definition}
For convenience, we omit the dependence on $\Bar{t}$ and only indicate the dependence on $p$ in the notation of a $p$-stagnation set. Indeed, the dependence on $\Bar{t}$ is of little importance, 
as $\Bar{t}$ is fixed and considered to be small compared to the total number of iterations $T.$

Thus with probability at least $p$, SGD will not leave the set $\mB$ anymore. Conditionally on this event, it is sufficient to control the Rademacher complexity of $\mB$ to bound the generalization gap.
On the contrary, a set is a $0$-stagnation set if the output of the Gaussian SGD algorithm does not stagnate in the set.
In this case, it is impossible to derive a generalization gap bound based on the Rademacher complexity of such a set. 
According to the empirical study in \cite{nagarajan2019uniform}, SGD does not stagnate in the sets $\Tilde{B}$ or $\check{B}$ defined in Section \ref{sec.backg}, implying that those are $p$-stagnation set for some small $p$ or even $p=0$.

\begin{figure}
    \centering
    \includegraphics[width=0.78\hsize]{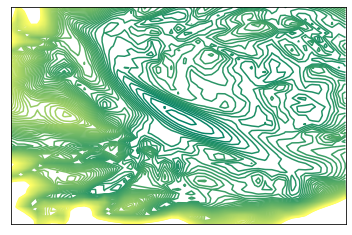}\\
    \includegraphics[width=0.78\hsize]{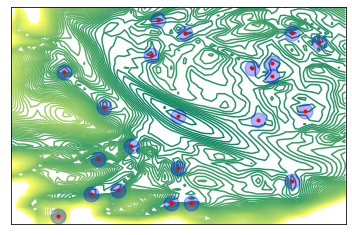}
    \caption{(Top) Contours of a two-dimensional projection of the test loss surface of a deep residual neural network with $56$ layers trained using CIFAR-10 data. For the training pruning of shortcuts is applied \cite{he2016deep}. 
    The two-dimensional projection is generated by dimension reduction of the $850,000$ dimensional parameter space using the random-direction method \cite{li2018visualizing}.
    (Bottom) Population minima (red dots) 
    and their neighbourhoods (blue balls) on the loss surface. Here the population minima are computed using the test data. For the visualization, the radius of all balls is set to $0.08$ in the dimension-reduced space.
    }
    \label{fig:non_convex}
\end{figure}

As illustrated in Figure \ref{fig:non_convex}, the loss landscape of neural networks is highly-nonconvex and possesses multiple local minima whose neighborhoods form attractive basins. Note that in the definition above, a stagnation set is non-random. As candidate for good stagnation sets, we take the union over local neighborhoods of the minima of the expected loss surface.

\begin{definition}[Population Minimum] \label{def:pop_minima}
We call a subset $\mu \subseteq \Theta$ a \textit{population minimum}, if $\mu$ is a maximally connected set consisting of local minima of $R(\mA)$.
\end{definition}
Let us emphasize again that those are minima of the expected loss surface and not the empirical loss surface. Consequently, the set $\mu$ is deterministic. To see why a minimum does not simply consists of one parameter as in the case of a strictly convex loss surface, consider a ReLU network with activation function $\sigma(x)=\max(x,0).$ Multiplying all parameters in one hidden layer of a ReLU network by $\alpha>0$ and dividing in another hidden layer all parameters by $1/\alpha,$ gives the same network function. Thus, if $\mA=(A_1,\ldots,A_L)\in \mu$ then also $\mA'=(\alpha A_1,\alpha^{-1}A_2,A_3,\ldots,A_L)\in \mu.$
For general activation function, if a minimum corresponds to a parameter $\mA=(A_1,\ldots,A_L)$ such that one of the weight matrices $A_\ell$ has a zero column vector, then, some of the parameters of the previous layer $A_{\ell-1}$ do not affect the loss function. Hence, the population minimum can in principle consist of several parameters, regardless of the choice of activation function.

In Definition \ref{def:neighbour}, we introduce a suitable notion of $\delta$-neighborhood around $\mu,$ denoted by $\mB_\delta(\mu) \subset \Theta.$ A key object in our approach is the union of $\delta$-neighbourhoods over several minima $\mu_1,...,\mu_K$ defined as
\begin{align}
    \mfB_{K,\delta} := \bigcup_{k=1}^K \mB_\delta(\mu_k).
    \label{eq.union_sets}
\end{align}
The main idea underlying this definition is that the Gaussian SGD parameter updates are attracted into these local neighborhoods and stagnate there. To prove a result in this direction, we impose \textit{local essential convexity} on the expected loss surface $R(\mA)$ within the neighbourhoods.
This assumption requires the loss function to be convex on a path between a parameter $\mA$ in the neighborhood and its projection on the local minimum $\mu$, as illustrated in Figure \ref{fig:projection_convex} (details are provided in Assumption \ref{asmp:essential_convex} below).
This condition can be used when the minimum consists of more than one parameter vector. It should be noted that the essential convexity is weaker than essential strong convexity \cite{liu2015asynchronous}, and similar to the Polyak-\L ojasiewicz condition \cite{polyak1963gradient}. Compared to these conditions, the local essential convexity seems better suited to derive bounds for the generalisation error in our framework.
When the SGD output stagnates in $\mfB_{K,\delta}$, we can use the Rademacher complexity of $\mfB_{K,\delta}$ to bound the generalization gap. Figure \ref{fig:non_convex} illustrates the non-convex loss landscape of a deep neural network by using dimension reduction. The displayed population minima and their neighborhoods were found using numerical methods.
\begin{figure}
    \centering
    \includegraphics[width=0.99\hsize]{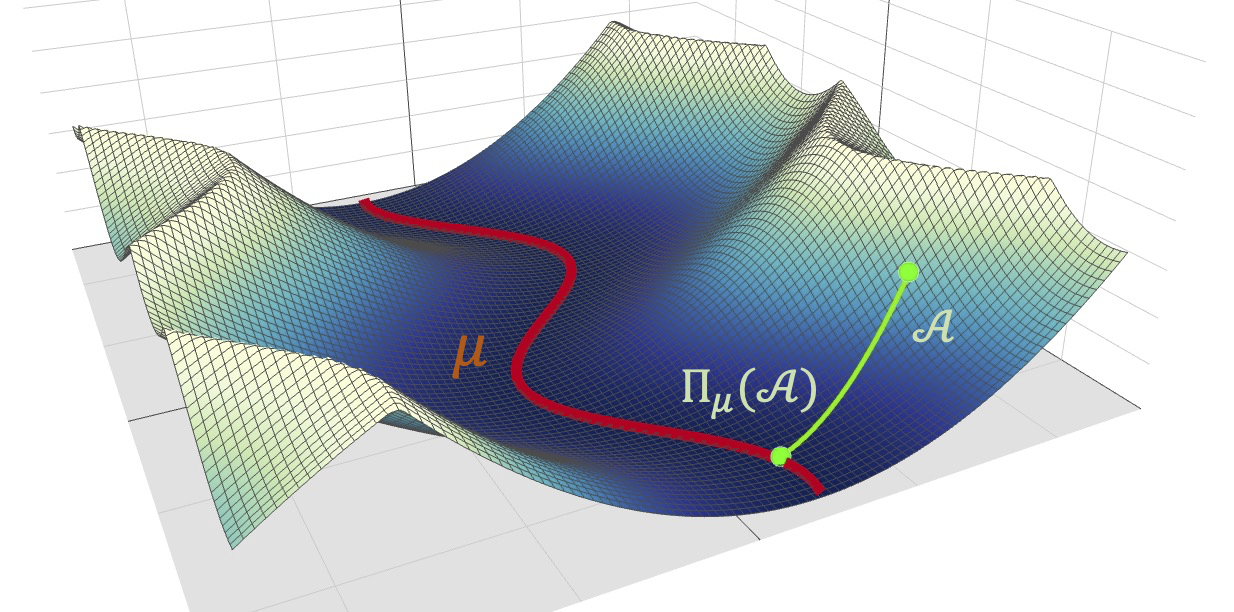}
    \caption{Illustration of the local essential convexity for the loss surface around (population) minima.
    The red curve denotes the minimum $\mu$, the green curve connecting $\mA \in \mB_\delta(\mu)$ and its projection $\Pi_\mu(\mA)$ is assumed to be convex.}
    \label{fig:projection_convex}
\end{figure}

In the described setting, we state two main results. The first provides a lower bound for the probability that the SGD iterations stagnate in $\mfB_{K,\delta}$ under several assumptions including local essential convexity.

\begin{theorem}[Informal Statement of Theorem \ref{thm:stay_prob_regularized}]
    Consider Gaussian SGD with arbitrary initialization (not necessary in $\mfB_{K,\delta}$) and a learning rate $\eta_t \asymp t^{-\alpha}$ for some $\alpha \geq 1$.
    Under a number of regularity assumptions and sufficiently large $n$ and $T$, there exist constants $c,c'> 0$ such that for any $\varepsilon \in (0,1/2),$ 
    $\mfB_{K,\delta}$ is an $(1-\varepsilon)$-stagnation set.
\end{theorem}
This result shows that under appropriate conditions, the union of population minima neighborhoods is a valid stagnation set. This result can be applied to generic loss surfaces and does not exploit the specific structure of neural networks.

The second main result is a generalization gap bound for deep neural networks if stagnation occurs. The statement can be summarized informally as
\begin{theorem}[Informal Statement of Theorem \ref{thm:bound_gen}]
    If Gaussian SGD stagnates in $\mfB_{K,\delta}$, then,
    \begin{align*}
         R(\Hat{\mA}) - R_n(\Hat{\mA}) =  \tilde{O}_\Pr\left(\frac{\sqrt{d \log K} + SL \delta \log^{1/2}(\Bar{h})}{\sqrt{n}} \right),
    \end{align*}
    where $d$ denotes the maximum dimension of the population minima $\mu_1, \dots, \mu_K$, $\Bar{h}$ is the maximum width of the deep neural network, and $S := \sup_{\mA \in \mfB_{K,\delta}}\prod_{\ell = 1}^L \|A_\ell\|_s$ with spectral norm $\|\cdot\|_s$.
\end{theorem}
This generalization gap bound depends on the network depth, the spectral norm of the weight matrices, and the radius $\delta$ of the minima neighbourhoods. While the number of network parameters $D$ does not explicitly appear in the bound, in practice, all quantities might depend in a highly non-trivial way on $D.$

On the technical side, we have developed three techniques to achieve the results: (I) evaluating the reaching probability of Gaussian SGD to the population minima neighborhood, (II) studying the probability of Gaussian SGD to stay in the neighbourhood of the population minima, and (III) metric entropy evaluation on neural networks within the population minima neighborhood.
For (I) reaching probability, we control the transition probability of the Gaussian SGD parameter update. For (II) the staying power of SGD, we apply the local essential convexity and locally uniform convergence of loss surfaces to show that the SGD stays in the neighborhood of population minima with high probability as the learning rate decreases. For (III) the entropy evaluation of neural networks, we combine the recent entropy analysis for deep neural networks in \cite{bartlett2017spectrally} with the uniform convergence tools for loss landscapes in \cite{mei2018landscape}.

\subsection{Related Work and Comparison} \label{sec:related_works}
This article is closely related to the work on SGD induced implicit regularization. We already briefly mentioned several articles in the previous section, and provide here a more in-depth overview. The argument that learning algorithms lead to implicit regularization has been shown in various settings such as matrix factorization, linear convolutional network, logistic regression, and others \cite{soudry2018implicit,ji2018risk,gunasekar2018characterizing,gunasekar2017implicit,gunasekar2018implicit,nacson2019convergence,li2018algorithmic}.
Although a precise understanding of implicit regularization in deep neural networks is still lacking, several articles assume that a norm of the parameters is bounded and give refined upper bounds on the generalization gap via uniform convergence \cite{neyshabur2015norm,neyshabur2017exploring,bartlett2017spectrally,golowich2017size,cao2020generalization}. A similar approach has been used in \cite{hardt2016train, kuzborskij2018data} based on the notion of algorithmic stability. The recent empirical work \cite{nagarajan2019uniform} shows, however, that the bounded norm assumption is violated.

Another way to evaluate the generalisation error is to introduce compressibility \cite{arora2018stronger,suzuki2019compression,li2020understanding}. These papers consider settings where the original neural network can be compressed into a neural network with  smaller capacity. By studying then uniform convergence for the compressed network class, tighter bounds for the generalisation gap are derived. \cite{arora2018stronger} compresses parameter matrices of a neural network by random projection and obtained an upper bound that can be evaluated in terms of the number of parameters of the compressed network.

The information-theoretic approach evaluates the generalization error using several discrepancies between measures such as mutual information \cite{russo2016controlling,xu2017information}. 
This method is well suited to analyze stochastic gradient Langevin dynamics (SGLD), which is a discretization of the Langevin dynamics. \cite{pensia2018generalization,bu2020tightening} derive generalization error bounds for SGLD, and \cite{negrea2019information} provides a refined analysis leading to data-dependent bounds.
In the same framework, \cite{neu2021information} considers a local perturbation of parameters and \cite{aminian2021exact} studies the Gibbs algorithm.
Although the SGLD is very similar to Gaussian SGD, the information-theoretic method is different from our approach.
Whereas we bound the generalization error by the size of stagnation sets using Rademacher complexity, the information-theoretic approach specifies parameter distributions by SGLD and evaluates generalization errors using information measures.

For non-convex optimization, bounds for the generalization error for SGD and its variations have been derived in \cite{raginsky2017non,mandt2016variational,jastrzkebski2017three,he2019control,kleinberg2018alternative,cheng2020stochastic}.
\cite{raginsky2017non,cheng2020stochastic} investigate the invariant distribution of stochastic differential equations. A particularly convenient approach is to employ Langevin dynamics in continuous time, in order to investigate global parameter search and to bound the generalization error \cite{raginsky2017non, mou2018generalization, tzen2018local}.
Although the analysis using invariant distributions is useful, it remains unclear whether those invariant distributions exist in general deep neural networks, since their existence requires specific assumptions.
Another direction is to study the local behaviour of SGD around local minima associated with loss shape  \cite{jastrzkebski2017three,he2019control,kleinberg2018alternative,liu2020loss}.
A limitation of this method is that only local properties can be investigated.

Another line of research is to study the generalization analysis of neural networks and related models in the overparametrized regime. When the number of parameters in the models is excessively large, there are multiple techniques to precisely measure generalization errors. To name a few, the spectrum-based analysis \cite{ belkin2019reconciling,mei2019generalization,liang2020just, montanari2019generalization,ba2019generalization,dobriban2018high, bartlett2020benign, tsigler2020benign}, and the use of loss functions whose shapes are almost convex or approaches zero due to a large number of network parameters \cite{allen2019learning,allen2019convergence,jacot2018neural}. 
A disadvantage of this approach is that until now it can only deal with linear or two-layer neural network models.

In contrast to earlier work, our approach aims to shed some light on the implicit regularization of Gaussian SGD induced by the loss surface and the geometry of its local minima. This source of implicit regularization arises from the structure of the deep networks and is not due to overparametrization of the model. Moreover, we consider the global behavior of SGD and allow the initial value of the algorithm to be far from the learned parameters.

\subsection{Notation}
For real numbers $a,a'$, $a \vee a' := \max\{a,a'\}$ is the maximum. 
For $d \in \N$, $I_d$ denotes an identity matrix of size $d \times d$.
The Euclidean norm of a vector $b \in \R^d$ is denoted by $\|b\|_2 := \sqrt{b^\top b}.$ For a matrix $A \in \R^{d \times d'}$, we define the matrix norms $\|A\|_{p,q} = \|(\|A_{:,1}\|_p,...,\|A_{:,h}\|_p)\|_q$ and write $\|A\|_s$ for the spectral norm (the largest singular value of $A$). Moreover, for the $L$ weight matrices $\mA =(A_1,\dots,A_L)$ in a deep network, we introduce the norms $\|\mA\|_{L,2,1} := \sum_{\ell = 1}^L \|A_\ell\|_{2,1}$ and $\|\mA\|_F^2 := \sum_{\ell =1}^L \|A_\ell\|_{2,2}^2$. We define $\langle \mA, \mA' \rangle := \sum_{\ell =1}^L \langle A_\ell, A'_\ell \rangle,$ where $\langle \cdot, \cdot \rangle$ denotes the Frobenius inner product. For $\mA$, let $\Pi_\mB(\mA)$ be the projection of $\mA$ onto the set $\mB$ with respect to the norm $\|\cdot\|_{L,2,1}$. In the following, let $C_{w}$ be a positive finite constant depending on $w$. For a set $\mB \subset \R^D$, $\lambda(\mB)$ denotes the Lebesgue measure of $\mB$.
For sequences $\{a_n\}_n$ and $\{b_n\}_n$, $a_n \lesssim b_n$ means that there exists $C>0$ such that $a_n \leq  C b_n$ holds for every $n \in \N$. Moreover, $a_n \gtrsim b_n$ iff $b_n \lesssim a_n$. Finally, we write $a_n \asymp b_n$ if both $a_n \gtrsim b_n$ and $a_n \lesssim b_n$ hold. We write $a_n = \Omega (b_n)$ for $\limsup_{n \to \infty} |a_n/b_n| > 0$.
$\mone\{E\}$ denotes the indicator function. It is $1$ if the event $E$ holds and $0$ otherwise.
For a set $F$, a norm $\|\cdot\|$ and $\varepsilon > 0$, $\mN(\varepsilon, F, \|\cdots\|)$ denotes the $\varepsilon$-covering number of $F$ with respect to $\|\cdot\|$.

\section{Setting}
\subsection{Deep Neural Network}
Let $L$ denote the number of layers, and let $A_\ell \in \R^{h_\ell \times h_{\ell
- 1}}$ be the weight matrices for each $\ell = 1,...,L$, with intermediate dimensions $h_\ell$ for $\ell = 0,...,L$. For convenience, we consider neural networks with one output unit, that is, $h_L = 1$. Moreover, $\bar{h} := \max_{j=1,...,L} h_j$ denotes the maximum width and $D= \sum_{\ell=1}^L h_\ell h_{\ell - 1}$ is the number of network parameters. For an activation function $\sigma(\cdot)$ that is $1$-Lipschitz continuous,
we define an $L$-layer neural network for $x \in \R^{h_0}$ as the function
\begin{align*}
    f(x; \mA) = A_L \sigma\big(A_{L-1} \cdots  \sigma(A_1 x) \cdots \big),
\end{align*}
with parameter tuples $\mA = (A_1,...,A_L) \in \Theta \subset \R^{D} := \R^{h_1 \times h_0} \times \R^{h_2 \times h_1} \times \cdots \times \R^{h_L \times h_{L-1}}$.

\subsection{Supervised Learning Problem} \label{sec:setting_supervised_learning}
Given observations $(X_i,Y_i) \in \mX \times \mY$ with sample spaces $\mX \subset \R^{h_0}$ and $\mY \subset \R^{h_L}$, we define the loss function $\ell:\mY^2 \to [0,1]$ as $\ell(Y_i,f(X_i,\mA)).$ To apply gradient descent methods, we need to assume that $\ell(Y_i,f(X_i,\mA))$ is differentiable with respect to the parameters. For such a loss and the sample $\mD_n := \{(X_i,Y_i)\}_{i=1}^n$, we define the corresponding empirical risk function $R_n(\mA):= n^{-1} \sum_{i=1}^n \ell(Y_i,f(X_i,\mA))$ and the expected risk function $R(\mA) := \Ep_{X,Y}[\ell(Y,f(X,\mA))]$.
To calculate derivatives of $R_n(\mA)$ and $R(\mA)$ with non-differentiable $\sigma$, we can take the sub-derivative of $\sigma$ instead, e.g., for the ReLU activation case, $\nabla \sigma(x) = \nabla \max\{x,0\} = \mone\{x \geq 0\}$.
We set $\mX = [0,1]^{h_0}$. This ensures $\sup_{x \in \mX} \|x\|_2 \leq \sqrt{h_0}$. 

\subsection{Gaussian Stochastic Gradient Descent}
\label{sec.GSGD_intro}
We study stochastic gradient descent (SGD) with Gaussian noise based on the empirical loss $R_n(\mA)$ to learn the parameters $\hat{\mA} \in \Theta.$ Let $\mA_0 \in \Theta$ be the initial parameter values and denote by $\mA_t \in \Theta$ for $t = 0,1,2,...$ the parameter values after the $t$-th SGD iteration.
We assume that $\Theta$ is compact and constraint the iterates of the algorithm to parameter tuples $\mA = (A_1,...,A_L) \in \Theta$. We assume existence of a filtration $\{\mF_t\}_{t \geq 0}$ and define an $\mA$-dependent $\mF_{t+1}$-measurable Gaussian vector $U_{t+1}(\mA)$ such that $\Ep[U_{t+1}(\mA)\mid \mF_t] = 0$ and  $\Cov(U_{t+1}(\mA) \mid \mF_t) = G(\mA)$ for all $ \mA \in \Theta$.
Here, $G: \Theta \to \R^{D}\times \R^{ D}$ is a matrix-valued map and all eigenvalues of $G(\mA)$ are assumed to be bounded from above by $c_G > 0$ and from below by $c_G' > 0$ for all $\mA \in \Theta$.
Given an initial parameter $\mA_0 \in \Theta$, we define a sequence of parameters ${\mA}_1,\mA_2,...,{\mA}_T$ by the following SGD update:
\begin{align}
    {\mA}_{t+1} = {\mA}_t - \eta_t \nabla R_{n}({{\mA}_t}) + \frac{\eta_t}{\sqrt{m}} U_{t+1}(\mA_t),\label{def:reguralized_SGD}
\end{align}
for $t=0,1,...,T-1$ with $\eta_t > 0$ the learning rate.
Here $T$ is a pre-determined deterministic stopping time $T$ and the output of the algorithm is $\hat{\mA} := \mA_T.$ If $\mA_{t+1}$ falls outside the parameter space $\Theta$, we project the vector onto $\Theta$ with respect to the $\|\cdot\|_F$-norm and then substitute it into $\mA_{t+1}$.
Furthermore, we assume that the initial point $\mA_0$ and the set of observations $\mD_n$ is $\mF_0$-measurable.

Gaussian SGD is an approximation of the widely used minibatch SGD. Theoretical and experimental similarities and differences are discussed in Section \ref{sec:minibatch_SGD}.

\section{Key Notion and Assumption}

\subsection{Population Minima and their Neighborhoods}

We introduce key concepts and assumptions related to the notion of population minima introduced in Definition \ref{def:pop_minima}. As we are considering the expected loss, everything is deterministic.

For a positive integer $d$ and a constant $c_{\mu} > 0$, we pick all population minima $\mu$ that satisfy for any $\varepsilon > 0$
\begin{align}
    &\mN(\varepsilon, \mu, \|\cdot\|_F) \leq c_\mu \varepsilon^{-d},
    \label{eq.d_def}
\end{align}
and that are separated from the boundary of the parameter space $\Theta$ with respect to the distance induced by the norm $\|\cdot\|_{L,2,1}$. Let $K$ be the number of population minima $\mu_1,...,\mu_K$ satisfying the conditions.
In the previous formula, $d$ can be viewed as a constraint on the dimension of the minima. 
If $\mu_1,...,\mu_K$ are isolated points, the inequality holds with $d = 0$. 
It should be noted that $d$ is not as large as the number of parameters $D$ (e.g. \cite{li2018measuring}).
For each $\mu \in \{\mu_1,...,\mu_K\}$, we consider a suitable notion of neighbourhood.
\begin{definition}[$\delta$-Neighborhoods of $\mu$] \label{def:neighbour}
    For $\delta > 0$, we define the $\delta$-neighborhood of $\mu$ as
\begin{align*}
    \mB_\delta(\mu) := \left\{\mA \in  \Theta \mid \inf_{\mA^* \in \mu}\|\mA - \mA^*\|_{L,2,1} \leq  \delta \right\}.
\end{align*}
This neighborhood is defined via the $\|\cdot\|_{L,2,1}$-norm, which is suitable for the generalization analysis of deep neural networks as developed in \cite{bartlett2017spectrally}.
Since the minima are assumed to be in the interior of the parameter space, we can find a sufficiently small $\delta$ such that $\mB_\delta(\mu_k) \subset \Theta$ holds for all $k=1,...,K.$ A key object in the analysis is the union of the $\delta$-neighborhoods
\begin{align*}
    \mfB_{K,\delta} := \bigcup_{k=1}^K \mB_\delta(\mu_k).
\end{align*}
\end{definition}

\subsection{Loss Surface and Gradient Noise}
We firstly discuss the local shape of the expected loss $R(\mA)$. Denote by $\Pi_\mu(\mA)$ the projection of $\mA$ onto a set $\mu$ with respect to the norm $\|\cdot\|_{L,2,1}$.

\begin{assumption}[Local Essential Convexity] \label{asmp:essential_convex}
For any $\mu \in \{\mu_1,...,\mu_K\}$ and any two parameters $\mA,\mA' \in \mB_\delta(\mu)$ with $\Pi_\mu(\mA) = \Pi_\mu(\mA')$, we have
\begin{align*}
    R(r\mA + (1-r)\mA') \leq rR(\mA) + (1-r) R(\mA'),
\end{align*}
for all $0\leq r\leq 1.$
\end{assumption}
This is a convexity property on the path between a parameter and its projected version in $\mB_\delta(\mu)$ as displayed in Figure \ref{fig:projection_convex}. The condition is weaker than local essential strong-convexity \cite{liu2015asynchronous}.
Since this version of convexity is defined for each projection path, Assumption \ref{asmp:essential_convex} does not depend on the shape of the population minima $\mu$: it is valid even if the minima $\mu$ is not isolated nor a non-convex set. 

There is no clear connection between local essential convexity and the Polyak-\L ojasiewicz (PL) condition \cite{polyak1963gradient}. While both conditions have a similar flavour, PL imposes a lower bound on the Frobenius-norm of $\nabla R(\mA)$ by a difference $R(\mA) - \inf_{\mA' \in \mB_\delta(\mu)} R(\mA')$ up to constants.

Existing work indicates that local essential convexity holds for neural networks. For example, \cite{kawaguchi2016deep,choromanska2015loss,kawaguchi2020elimination} have shown that for specific neural network architectures the Hessian matrix around minima is positive definite. This implies local essential convexity. The theory of neural tangent kernels \cite{jacot2018neural} with over-parametrized neural networks yields the same result.
We also mention that another notion of convexity has been shown by \cite{safran2016quality,milne2019piecewise} in connection with regularization and initial values.

Moreover, we need to impose some smoothness on the expected loss surface $\mA\mapsto R(\mA)$. Let $\nabla_{\ell,j,k}$ denote the partial derivative with respect to the $(j,k)$-th entry of the weight matrix $A_\ell$.
\begin{assumption}[Local Smoothness] \label{asmp:local_smoothness}
For any $\mu \in \{\mu_1,...,\mu_K\}$ and any parameters $\mA, \mA' \in \mB_\delta(\mu)$, $R$ is differentiable at $\mA$ and we have
\begin{align*}
    &|\nabla_{\ell,j,k} R(\mA) - \nabla_{\ell,j,k} R (\mA')| \leq C_{\nabla R} \|\mA - \mA'\|_F, 
\end{align*}
for all $\ell=1,...,L$ and all $(j,k) \in \{1,...,h_{\ell - 1}\} \times \{1,...,h_{\ell}\}$. 
\end{assumption}
Assumptions of this type are common, see e.g. \cite{raginsky2017non}. It should be noted that the condition only has to hold locally in the neighborhoods of the minima $\mfB_{K,\delta}.$ Since the assumption is about the expected loss, it holds for non-smooth activation functions such as the ReLU activation function. 
To describe this fact in more detail, we provide a simple example with few parameters and layers.
\begin{lemma}\label{lem:relu}
    Consider a neural network with two layers $L=2,$ widths $h_0=h_1=h_2=1$, ReLU activation $\sigma(x)=\max\{x,0\}$, and square loss function $\ell(t,t')=(t-t')^2$.
    Suppose that the maximum absolute value of each element of $\mA$ is bounded by $B > 0$.
    Then, for some distribution of $(X,Y)$, Assumption \ref{asmp:local_smoothness} holds with $C_{\nabla R} = \sqrt{2}(1 + \tfrac{4}{3}B^2)$.
\end{lemma}

The constant $C_{\nabla R}$ will affect the SGD stagnation probabilities but does not influence the generalization gap in Theorem \ref{thm:bound_gen}.

We also impose an assumption on the effect of data-oriented uncertainty on gradient noise.
Let $Z = (Y,X)$ be a pair of the input and output variables.
\begin{assumption}[Gradient Noise] \label{asmp:sample_gradient_noise}
There exists $\xi > 0$ such that for any $\mu \in \{\mu_1,...,\mu_K\}$ and any parameter $\mA \in \mB_\delta(\mu)$, the following inequality holds: 
\begin{align*}
    \Ep_{X,Y}\left[ \exp \left( \big\langle b, \nabla \ell(Y, f(X, \mA)) - \nabla R(\mA) \big\rangle\right) \right] \leq \exp \left( \frac 12 \xi^2 \|b\|_F^2 \right).
\end{align*}
\end{assumption}
This condition is the main assumption in \cite{mei2018landscape} and allows us to establish uniform convergence of the generalization gap within the set $\mfB_{K,\delta}$.

\section{Main Result}

\subsection{Stagnation Probability}
We first introduce the lower bound on the stagnation probability for Gaussian SGD. Explicit expressions for all constants are given in the proofs. 
Recall that $\lambda(B)$ denotes the Lebesgue measure of the set $B \subset \R^D$, and $C_w$ denotes a finite positive constant depending on $w$.

\begin{theorem}[Stagnation Probability of Gaussian SGD] \label{thm:stay_prob_regularized}
    Consider the Gaussian SGD algorithm in \eqref{def:reguralized_SGD} and
    suppose Assumptions \ref{asmp:essential_convex}, \ref{asmp:local_smoothness}, and \ref{asmp:sample_gradient_noise} hold.
    For positive $\alpha,$ let the learning rates $\eta_t \asymp t^{-\alpha}$ be a monotonically decreasing sequence in $t$.
    Define $\underline{t} = \min\{t \mid \eta_t \leq C_{L, \Bar{h}, \nabla R} + C_{m,L,\Bar{h},D,G,\Theta, K}, t \geq C_{m, D,L,\Bar{h},G,\delta, \mfB}\}$.
    Then, for $\delta > C_{m,L,\Bar{h},D,G,\Theta}$
    and all sufficiently large $n$ and $\Bar{t} - \underline{t}$, the union of neighborhoods around local minima $\mfB_{K,\delta}$ is a $p^*$-stagnation set with
    \begin{align*}
        p^* \geq 1 -  (c_2 \lambda(\mfB_{K,\delta}))^{-1}\exp( - c_1 \Bar{t}^{2\alpha}) - (T - \Bar{t}) \exp(-c_3 \delta^2 \Bar{t} ),
    \end{align*}
    for $\Bar{t} \in [\underline{t}, T]$ and suitable constants $c_1 := C_{m,L,\Bar{h},D,G,\Theta}, c_2 := C_{m,D,G}$, and $c_3 := C_{ m, D, G,\Theta, \underline{t}, \alpha}$.
\end{theorem}

For fixed $\delta$ and sufficiently large $\Bar{t},$ the stagnation probability becomes arbitrarily close to one.

The lower bound of the stagnation probability consists of two main components:  the probability of entering the set $\mfB_{K,\delta}$, and the probability of not exiting $\mfB_{K,\delta}$ during the Gaussian SGD iterations. 
For any $\Bar{t} \in [\underline{t}, T]$, we can decompose the stagnation probability as follows:
\begin{align}
    &\Pr(\mA_t \in \mfB_{K,\delta}, \forall t \in  [\overline{t}, T]) \label{ineq:decomp_stag_prob} \\
    &=\Pr(\mA_t \in \mfB_{K,\delta}, \forall t \in  [\overline{t} +1 , T] \mid \mA_{\overline{t}} \in \mfB_{K,\delta}) \Pr(\mA_{\overline{t}} \in \mfB_{K,\delta}). \notag 
\end{align}
Here, $\Pr(\mA_{\overline{t}} \in \mfB_{K,\delta})$ describes the probability of staying $\mfB_{K,\delta}$ at time $\bar{t}$, and $\Pr(\mA_t \in \mfB_{K,\delta}, \forall t \in  [\overline{t} +1 , T] \mid \mA_{\overline{t}} \in \mfB_{K,\delta})$ is the non-exiting probability from $\mfB_{K,\delta}$. Theorem \ref{thm:stay_prob_regularized} follows by carefully bounding these terms from below. 

The decomposition into the two probabilities occurring on the right hand side of \eqref{ineq:decomp_stag_prob} allows us to link the analysis to the two phases that are commonly observed during the training process. Indeed, Gaussian SGD globally explores the parameter space first to get into a neighborhood of some local minimum. The parameter updates still can escape this neighborhood and also neighborhoods around local minima that are visited later with a certain probability. During this first training stage, Gaussian SGD can move far away from its initial values. To see this assume that the Gaussian contributions $U_{t+1}(\mA_t)$ in \eqref{def:reguralized_SGD} are multivariate standard normal. Then, we have that $\|U_{t+1}(\mA_t)\|_F^2$ is of the order $D$ with high probability. This then also means that $\|\mA_{t+1}-\mA_t\|_F^2$ is at least $O(\eta_t^2D/m).$ But as the number of iterations increases, the learning rate $\eta_t$ will get smaller, and Gaussian SGD will move slower and eventually stagnate at a local minimum. It is more likely that Gaussian SGD ends up in a wide local minimum as those are more difficult to escape from. Previous work studied the probabilistic escape from the neighborhood of a local minimum \cite{kleinberg2018alternative} and the stagnation in the neighborhood \cite{leclerc2020two} separately.

The lower bound in the previous theorem increases in $\Bar{t} \in [\underline{t}, T]$, but inequality \eqref{ineq:decomp_stag_prob} implies moreover that the parameter remains in the neighbourhood of the local minima for the last $T - \Bar{t}$ iterations. To know this is useful for situations such as fine-tuning or the stochastic weight averaging algorithm \cite{izmailov2018averaging}, where the parameter iterates need to have the same properties for a certain period of time.

As a consequence of the previous theorem, we obtain the following simplified result:
\begin{corollary} \label{cor:staying_prob}
Suppose $\underline{t} = 1$ and consider otherwise the same setting as in Theorem \ref{thm:stay_prob_regularized}. 
For any $\varepsilon \in (0,1)$, if $\Bar{t} \geq  \max\{c_1^{-1} \log ((\varepsilon^{-1 }- 1) / (c_2 \lambda(\mfB_{K,\delta})))^{1/2\alpha}, T^{1/2} / ( \varepsilon^{1/2} c_3 \delta^2 )  \}$, we have
\begin{align*}
    p^* \geq 1- \varepsilon.
\end{align*}
\end{corollary}
The imposed conditions essentially imply that $\Bar{t}$ should be of the order $O(\sqrt{T} \varepsilon^{-1/2})$.

\subsection{Generalization Gap Bound}
We provide an upper bound on the generalization error for deep neural networks, provided the SGD iterations remain in the stagnation set. As before, consider a deep neural network with $L$ layers, maximum width $\Bar{h}$, and $D$ network parameters.
We also define the largest possible product of its spectra within $\mfB_{K,\delta}$ by $S :=  \sup_{\mA \in \mfB_{K,\delta}} \prod_{\ell = 1}^L \|A_\ell\|_s$, and set $B:=\sup_{\mA \in \Theta}\|f(\cdot; \mA)\|_{L^\infty}$ for an $L^\infty$-upper bound over all neural network functions. 
The following result provides a bound on the generalization gap:
\begin{theorem}[Bound on Generalization Gap] \label{thm:bound_gen}
    If the same conditions as for Theorem \ref{thm:stay_prob_regularized} hold, then, for any $\varepsilon \in (0,1)$, with probability at least $p^* - \varepsilon$, we obtain
    \begin{align*}
        &R(\hat{\mA}) - R_n(\hat{\mA})\\
        &\leq C_{\mu,h_0,B} \frac{\sqrt{d(1+ \log K)} + SL \delta \log^{1/2}(\Bar{h}) \log n}{\sqrt{n}} + 3B\sqrt{ \frac{\log (1/\varepsilon)}{2n}}.
    \end{align*}
\end{theorem}
Observe that the bound holds with at least probability $p^*-\varepsilon$, where $p^*$ is the lower bound on the stagnation probability in Theorem \ref{thm:stay_prob_regularized}. In terms of network quantities, the number of network parameters $D$ does not affect the derived bound on the generalization gap directly. Instead, the depth $L$ and the product of spectral norms $S$ appear in the generalization gap, similar as in \cite{bartlett2017spectrally} without stagnation. In practice, $L$ is typically in the range up to a few hundred and can be viewed as bounded.
$S$ is not affected by the number of parameters and does not necessarily increase even with large models. It is also of interest to notice that the increase in the number of minima neighborhoods $K$ does not have a significant effect on the bound. This implies that if $\mfB_{K,\delta}$ is constituted from a larger number of neighborhood sets $\mB_\delta(\mu_k)$, the generalization gap increases moderately.

The number of network parameters $D$ does not appear explicitly in the bound, but affects it implicitly through the various parameters that can depend on $D$. Below we provide an example that makes the $D$ dependence more explicit.\\

\subsubsection{Example of Neural Network} \label{sec:example}
Applying Theorem \ref{thm:bound_gen} to a shallow ReLU network with most of the parameters concentrated in the first layer, we get more insights into the dependence on the number of network parameters $D.$

Consider the regression problem $\mY =\R$ with quadratic loss $\ell (y,y') = (y-y')^2.$ 
The neural networks have $L=2$ layers and $h$ units in both the input layer and the hidden layer, that is, $h_0=h_1=h$. Hence the total number of parameters is $D=h^2+h$.
Moreover, we consider ReLU activation function $\sigma(z) = \max\{z,0\},$ parameter space $\Theta = [- \Bar{c}, \Bar{c}]^D \subset \R^D$ for some $\Bar{c} \geq 1$, design distribution $X \sim N(0, I_h)$ and output $Y=f^*(X)= A_2^* \sigma (A_1^* X)$, where $\mA^* = (A_1^*, A_2^*)$ with $A_2^* = (1/\sqrt{h},...,1/\sqrt{h}) \in \R^{h}$ and $A_1^* = I \in \R^{h \times h}$ denote the true weight matrices.

Based on the previous theorem, we can bound the generalization gap in this setting. 
\begin{proposition} \label{prop:gap_example}
    Suppose that the same setup as in Theorem \ref{thm:stay_prob_regularized} holds for $\eta_t, \delta,$ and $\underline{t}$. Let $\varepsilon \in (0,1).$ Conditionally on the event that the Gaussian SGD stagnates in $\mB_\delta(\mA^*)$, the following inequality holds with probability at least $1-\varepsilon$:
    \begin{align*}
        &R(\hat{\mA}) - R_n(\hat{\mA})\\
        &\leq C_{\mu,B}  (\delta + \delta^3)\sqrt{\frac{h\log h}n}\log n+ 3B\sqrt{ \frac{\log (1/\varepsilon)}{2n}}.
    \end{align*}
\end{proposition}
The generalization gap is primarily described by the width parameter $h$ and the radius $\delta$.
If $\delta$ is large, the probability for stagnation is large and thus the inequality will hold with high probability. 
On the contrary, for small $\delta$, the upper bound becomes tighter but the probability that the inequality holds decreases. Since $h=O(\sqrt{D}),$ the $\sqrt{h}$ term in the bound is of the order $O(D^{1/4})$.

\subsection{Relation between Gaussian SGD and Minibatch SGD} \label{sec:minibatch_SGD}

We firstly discuss minibatch SGD, which is defined as follows. For each $t$, we randomly pick a minibatch of $m$ observations $(Y^{t}_j, X^t_j)_{j=1}^m$ from the full sample $\{(X_i,Y_i)\}_{i=1}^n$ and set $\hat{R}_n^t(\mA) := m^{-1} \sum_{j=1}^m \ell(Y_j^t, F_\mA(X_j^t))$.
The minibatch SGD algorithm generates the sequence $\{\tilde{\mA}_t\}_t$ by the following updating equation
\begin{align}
    \tilde{\mA}_{t+1} = \tilde{\mA}_t - \eta_t\nabla \hat{R}^t_n(\tilde{\mA}). \label{def:original_SGD}
\end{align}
For the filtration $\{\mF_t\}_{t \geq 0}$ introduced before, \textit{batch gradient noise} is defined as the $\mA$-dependent $\mF_{t+1}$-measurable random vector $W_{t+1}(\mA) = \sqrt{m}(\nabla R_n(\mA) -  \nabla \hat{R}_n^t(\mA) ).$ This random variable measures the effect of the subsampling on the gradient and allows to rewrite \eqref{def:original_SGD} as
\begin{align}
    \tilde{\mA}_{t+1} = \tilde{\mA}_t - \eta_t \nabla R_{n}({\mA_t}) + \frac{\eta_t}{\sqrt{m}} W_{t+1}(\tilde{\mA}_{t}). \label{def:simple_SGD}
\end{align}
We note that $\Ep[W_{t+1}(\mA)\mid \mF_t] = 0$ holds for all $ \mA \in \Theta$ and conditional on the observations, $W_{t+1}(\mA)$ is an $\mA$-dependent $\R^D$-valued random variable with zero mean and finite variance.

For large batch-size $m,$ mini-batch SGD and Gaussian SGD as defined in \eqref{def:reguralized_SGD} behave very similar. Indeed for large $m$, $W_{t+1}(\mA)$ follows asymptotically a Gaussian distribution by the conditional multiplier central limit theorem. More precisely, for given training dataset and the minibatch sampling regarded as independent multipliers, $W_{t+1}(\mA)$ weakly converges to a Gaussian law almost surely (e.g. Lemma 2.9.5 in \cite{van1996weak}). Empirically, several studies \cite{panigrahi2019non,csimcsekli2019tail,xie2020diffusion} investigate the tail behavior of minibatch SGD, and some of them report that gradient noise has Gaussian-like tail probabilities.

We suspect that the lower bound on the stagnation probability in Theorem \ref{thm:stay_prob_regularized} also holds for mini-batch SGD. 
This is because if the noise $W_{t+1}(\mA)$ of minibatch SGD and the noise $U_{t+1}(\mA)$ of the Gaussian SGD are sufficiently close, e.g. in the sense of the chi-square divergence, then by a Girsanov-type change of measure transformation \cite{girsanov1960transforming} one can show that the parameter updates are nearly the same.

As a comparison, we state the information theoretic bound in \cite{negrea2019information} for the generalization gap induced by SGLD. Considering mini-batch risk $\Tilde{R}_S(\mA) := $ $|S|^{-1} \sum_{(Y,X) \in S} \ell (Y, f(X;\mA))$ with randomly sampled data subset $S \subset \mD_n,$ the sequence of parameters learned by the SGLD process $\{{\mA}^{\mathrm{SGLD}}_t\}_{t=1}^T$ is given by
\begin{align*}
    {\mA}^{\mathrm{SGLD}}_{t+1} = {\mA}^{\mathrm{SGLD}}_t - \eta_t \nabla \tilde{R}_S ({\mA}^{\mathrm{SGLD}}_t) + \sqrt{\frac{2 \eta_t}{\beta_t}} E_t,
\end{align*}
where $\eta_t > 0$ is the learning rate, $\beta_t > 0$ is an inverse temperature parameter,  and $E_t$ is an independently drawn standard normal Gaussian random vector. For this scheme, \cite{negrea2019information} proves the inequality
\begin{align*}    \Ep\left[R_n({\mA}^{\mathrm{SGLD}}_T) - R({\mA}^{\mathrm{SGLD}}_T)\right] \leq C_+ \Ep \left[   \sqrt{\frac{\sum_{t=1}^T \beta_t \eta_t \mathrm{Trace} ( \hat{\Sigma}_t )}{n}} \right].
\end{align*}
This bound on the generalization gap decodes the gradient information through the gradient covariance matrix $\hat{\Sigma}_t := \Ep[ \Var(\nabla \Tilde{R}_S(\mA^{\mathrm{SGLD}}_t))].$ The inequality is structurally different from our approach based on stagnation sets and local uniform convergence and it remains unclear how it can be linked to the geometry of the local minima or the depth and width parameters in the network architecture.

\subsection{Application to Optimization Error Bound} \label{sec:optimization_bound}
We now apply the stagnation results obtained above to evaluate the generalization gap using Gaussian SGD with respect to the number of iterations. We pick a local minimum $\mu \in \{\mu_1,...,\mu_K\}$ and define the minimum value of the expected loss $R(\mA)$ in $\mB_\delta(\mu)$ as $R_* := \min_{\mA \in \mB_\delta(\mu)} R(\mA)$. Assuming that for fixed $\mu$ stagnation in $\mB_\delta(\mu)$ occurs, we obtain a new bound on the generalization gap.

Let $\Ep_{\mE}[\cdot]$ be the expectation conditionally on the event $\{\mA_t \in \mB_\delta(\mu), t \in [\Bar{t}, T]\}$. Recall that $m$ is as defined in \eqref{def:reguralized_SGD}.
\begin{theorem}[Optimization Error Bound] \label{thm:bound_gen_var}
    Suppose Assumptions \ref{asmp:essential_convex}, \ref{asmp:local_smoothness}, and \ref{asmp:sample_gradient_noise} hold and the learning rate is $\eta_t = c t^{-1}$ for some positive constant $c$. Then, conditionally on the event of stagnating in $\mB_\delta(\mu)$ with given $\mu$, we obtain
    \begin{align*}
        &\Ep_{\mE}\left[R(\mA_T) - R_*\right] \leq \frac{C_{\xi, \mu, R, \nabla R} \delta^2 (1 + m^{-1} + n^{-1/2})}{T}.
    \end{align*}
\end{theorem}

We find that the error linearly converges in $1/T$, which is the optimal rate for SGD under the Polyak-\L ojasiewicz condition \cite{karimi2016linear}. 
Moreover, the bound increases with the radius of the neighborhoods $\delta$. 
Comparing with Theorem \ref{thm:bound_gen}, there is no dependence on the product of spectral norms $S$. While the generalization gap in Theorem \ref{thm:bound_gen} is derived by uniform convergence on $\mfB_{K,\delta},$ the derivation for Theorem \ref{thm:bound_gen_var} depends on the gradient of the loss function associated with the deep neural networks.

The value $R_*$ is by definition the optimal attainable risk (oracle risk) over all network parameter choices in $\mB_\delta(\mu).$ The size of this quantity depends on the approximation capabilities of the network, which we do not study in this work. As completely different techniques are used, separating expressive power and generalization error is common in the theoretical neural networks literature.

\subsection{Proof Outline}

We provide an overview of the proofs for Theorem \ref{thm:stay_prob_regularized} and Theorem \ref{thm:bound_gen}. 
Full proofs are given in a later section.

\subsubsection{Stagnation Probability (Theorem \ref{thm:stay_prob_regularized})}

As in \eqref{ineq:decomp_stag_prob}, we decompose the stagnation probability into the reaching probability $\Pr(\mA_{\overline{t}} \in \mfB_{K,\delta}),$ and the non-escaping probability $\Pr(\mA_t \in \mfB_{K,\delta}, \forall t \in  [\overline{t} +1 , T] \mid \mA_{\overline{t}} \in \mfB_{K,\delta})$.
For both terms, we derive a lower bound.

\textbf{(i) Reaching probability}:
We first derive a recursive formula for the ratio
\begin{align*}
    r_t :=  \frac{\Pr(\mA_{{t}} \in \mfB_{K,\delta})}{\Pr(\mA_{{t}} \notin \mfB_{K,\delta})}.
\end{align*}
Denoting the conditional probability for transitions by $p_t := \Pr(\mA_{{t}} \in \mfB_{K,\delta} \mid \mA_{{t}-1} \notin \mfB_{K,\delta})$ and $q_t := \Pr(\mA_{{t}} \notin \mfB_{K,\delta} \mid \mA_{{t}-1} \in \mfB_{K,\delta})$, it follows that
\begin{align}
    r_t 
    &= \frac{(1-q_t) r_{t-1} + p_t}{q_t r_{t-1} + 1-p_t} =: F(r_{t-1} ; p_t,q_t). \notag
\end{align}
Using properties of the map $F (\cdot; p_t, q_t)$, we can deduce that $r_{\Bar{t}}$ is larger than $r_{\underline{t}}^* - \varepsilon$ for a small $\varepsilon > 0$ and $r_{\underline{t}}^*$ a fixed point $r_{\underline{t}}^* = F(r_{\underline{t}}^*; p_{\underline{t}}, q_{\underline{t}})$. This is the key argument that allows us to finally obtain the lower bound
    \begin{align*}
        \Pr(\mA_{\bar{t}} \in \mfB_{K,\delta}) = \frac{r_{\Bar{t}}}{1 + r_{\Bar{t}}}  \geq 1 - \frac{\exp(  - C_{m,L,\Bar{h},D, \delta, G, \Theta} \bar{t}^{2\alpha})}{C_{m,D,G}\lambda (\mfB_{K,\delta})},
    \end{align*}
for sufficiently large $\overline{t}$.

\textbf{(ii) Non-escaping probability}:
To derive a lower bound on $\Pr(\mA_t \in \mfB_{K,\delta}, \forall t \in  [\overline{t} +1 , T] \mid \mA_{\overline{t}} \in \mfB_{K,\delta})$, we decompose the expression into a step-wise non-escaping probability. Picking one local minima $\mu \in \{\mu_1,...,\mu_K\}$, we find 
\begin{align}
    &\Pr \big( \mA_t \in \mB_\delta(\mu) , \forall t \in \{\tau+1,...,T\} \mid \mA_{\tau }\in \mB_\delta(\mu)\big) \notag \\
    &= \prod_{t = \tau + 1}^T \Pr \big( \mA_t \in \mB_\delta(\mu) \mid \mA_{t-1} \in \mB_\delta(\mu) \big). \label{eq:outline_stay}
\end{align}
We now evaluate the step-wise conditional staying probability $\Pr \left( \mA_t \in \mB_\delta(\mu) \mid \mA_{t-1} \in \mB_\delta(\mu) \right).$ To this end, we define the updated parameter without gradient noise as $\mA_{t}^- := \mA_{t-1} - \eta_{t - 1} \nabla R_n(\mA_{t - 1})$ and observe that $\mA_t = \mA_{t}^- + \frac{\eta_{t-1}}{\sqrt{m}} U_{t}(\mA_{t-1})$ holds, implying
\begin{align}
    &\Pr \left( \mA_t \in \mB_\delta(\mu) \mid \mA_{t-1} \in \mB_\delta(\mu) \right) \notag \\
    &\geq \Pr \left( \mA_t \in \mB_\delta(\mu) \mid \mA_{t}^- \in \mB_\delta(\mu) \right)  \notag \\
    & \quad \times \Pr \left( \mA_t^- \in \mB_\delta(\mu) \mid \mA_{t-1} \in \mB_\delta(\mu) \right). \label{eq:outline_stay_step}
\end{align}
Lemma \ref{lem:local_attraction} and Assumption \ref{asmp:essential_convex} yield the lower bound $\Pr \left( \mA_t^- \in \mB_\delta(\mu) \mid \mA_{t-1} \in \mB_\delta(\mu) \right) \geq 1- \varepsilon$ with sufficiently small $\varepsilon > 0$.
Then, for $\mA_t^- \in \mB_\delta(\mu)$, we obtain
\begin{align*}
    &\Pr \left( \mA_t \in \mB_\delta(\mu) \mid \mA_{t}^- \in \mB_\delta(\mu) \right) \\
    & \geq 1 - \Pr \left(\mA_{t}^- + m^{-1/2}{\eta_{t-1}}W_{t}(\mA_{t-1}) \notin \mB_\delta(\mu)\mid \mA_{t}^- \in \mB_\delta(\mu) \right).
\end{align*}
Since $m^{-1/2}{\eta_{t-1}}W_{t}(\mA_{t-1})$ is Gaussian, we can control its tail behavior.
Combining this result with \eqref{eq:outline_stay_step} and \eqref{eq:outline_stay} yields the lower bound for the non-escaping probability.

\textbf{(iii) Combining the result}: 
Combining the previous results yields Theorem \ref{thm:stay_prob_regularized}.

\subsubsection{Generalization Gap Bound (Theorem \ref{thm:bound_gen})}
This upper bound is based on uniform convergence over the union of neighborhoods around local minima $\mfB_{K,\delta}$. Suppose that the event $\hat{\mA} \in \mfB_{K,\delta}$ holds, whose probability is guaranteed to be no less than $p^*$ by Theorem \ref{thm:stay_prob_regularized}.
For a function class $\mF$, we consider the standard definition of the empirical Rademacher complexity $\fR_n(\mF) := \Ep_{u_{1:n} | X_{1:n}}\left[\sup_{f \in \mF} \frac{1}{n} \sum_{i=1}^n u_i f(X_i) \right]$ with independent Rademacher variables $u_i, i=1,...,n$, that is, both $u_i = 1$ and $u_i=-1$ have probability $1/2.$
Then, for sufficiently small $\varepsilon > 0$, we obtain
\begin{align*}
    &R( \hat{\mA}) - R_n(\hat{\mA}) \\
    & \leq \sup_{\mA \in \mfB_{K,\delta}} R(\mA) - R_n(\mA) \\
    &\leq 2 \fR_n\big( \{\ell (\cdot, f(\cdot; \mA)) \mid \mA \in \mfB_{K,\delta}\} \big) + 3B\sqrt{ \frac{\log (1/\varepsilon)}{2n}}, 
\end{align*}
with probability at least $p^*-\varepsilon$.
To bound the Rademacher complexity term, we evaluate the covering number of the function class $\mF(\mfB_{K,\delta}\mid X_{1:n}) := \{\{f(X_i;\mA)\}_{i=1}^n \mid \mA \in \mfB_{K,\delta}\}$ with respect to the distance induced by the norm $\|\cdot\|_F$.
In Proposition \ref{prop:covering}, with $\varepsilon > 0$, we obtain
\begin{align}
    &\log\mN\big(\varepsilon, \mF(\mfB_{K,\delta}\mid X_{1:n}),\|\cdot\|_F \big) \notag \\
    &\leq C_\mu d \log\Big(\frac K{\varepsilon}\Big)  + \frac{ n h_0 S^2 L^2 \delta^2 \log (2 \Bar{h}^2) }{\varepsilon^2}. \label{ineq:outline_covering}
\end{align}
Since the number of parameters $D$ does not appear in this covering inequality, we obtain the claim of Theorem \ref{thm:bound_gen}.

The bound on the covering number is based on two key ingredients. Firstly, we apply the spectrum-based recursive covering developed by \cite{bartlett2017spectrally}. This technique evaluates the covering of neural network functions using the spectral norm of the parameter matrix $A_\ell$ in the $\ell$-th layer. More precisely, let $\varepsilon_a, \varepsilon_x > 0$ be fixed values, $\check{X}_{\ell - 1}, \check{X}_{\ell - 1}' \in \R^{h_{\ell -1}}$ be two possible outputs of the $(\ell-1)$-st layer satisfying $\|\check{X}_{\ell - 1} - \check{X}_{\ell - 1}'\|_2 \leq \varepsilon_x$, and $A_\ell, A_\ell' \in \R^{h_{\ell - 1} \times h_\ell}$ be parameter matrices for the $\ell$-th layer such that $\|A_\ell - A_\ell'\|_F \leq \varepsilon_A$.
Then, the discrepancy between two outputs of the $\ell$-th layer is bounded by
\begin{align*}
    \big\|\sigma (A_{\ell} \check{X}_{\ell - 1}) - \sigma (A_{\ell}' \check{X}_{\ell - 1}') \big\|_2 &\leq \|A_{\ell} \check{X}_{\ell - 1} - A_{\ell}' \check{X}_{\ell - 1}'\|_2 \\
    &\leq \|A_j\|_s \|\check{X}_{\ell - 1} - \check{X}_{\ell - 1}'\|_2 + \varepsilon_a \\
    &\leq \|A_j\|_s \varepsilon_x + \varepsilon_a,
\end{align*}
whenever $\sigma$ is a $1$-Lipschitz continuous activation function. Applying this inequality, we can cover the output space of the $\ell$-th layer recursively, using the covering of the output space of the previous layer and a covering of the weight matrices in the $\ell$-th hidden layer. Working inductively layer by layer, we obtain the final bound on the covering number of deep neural networks.

We also need a covering bound for the union of neighborhoods of local minima $\mfB_{K,\delta}$. For that we apply the matrix covering bound in \cite{zhang2002covering}. Since the ball $\mB_\delta(\{\mA\})$ is taken with respect to the distance induced by the norm $\|\cdot\|_{L,2,1}$, we can derive a covering bound which does not depend on the number of network parameters.
We extend this bound to neighbourhoods of a local minimum $\mu,$ proving in Lemma \ref{lem:covering_neighbour} that
\begin{align*}
    \log \mN \big(\varepsilon, \mB_\delta(\mu), \|\cdot\|_F\big) \leq C_\mu d \log \Big(\frac 1{\varepsilon}\Big) + \frac{ L \delta^2 \Bar{h}^2}{\varepsilon^2} \log \big(2 \Bar{h}^2\big).
\end{align*}
Both covering bounds combined yield the covering inequality \eqref{ineq:outline_covering}. Together with Dudley's theorem the proof of Theorem \ref{thm:bound_gen} can be completed.

\subsubsection{Optimization Error Bound (Theorem \ref{thm:bound_gen_var})}

The proof of Theorem \ref{thm:bound_gen_var} is obtained by a combination of the derived covering bound in Lemma \ref{lem:covering_neighbour} and the convergence analysis under the Polyak-\L ojasiewicz condition (e.g. Theorem 1 in \cite{karimi2016linear}). While in the seminal convergence analysis by \cite{mei2018landscape}, the uniform convergence error of the empirical loss landscape depends on the number of parameters $D$, we consider convergence of the local loss surface within $\mfB_{K,\delta}$. Because the entropy bound for the set $\mfB_{K,\delta}$ does not depend directly on $D,$ we obtain the error bound in Theorem \ref{thm:bound_gen_var}.

\section{Simulation Study}

We study empirically the effect of the number of parameters $D$, the radius $\delta$ of the neighborhood around the local minima, and the stagnation probability $p^*$ on the generalization gap in the neural network example in Section \ref{sec:example}. 

For each run, we generate $n=100$ independent training samples with true parameter matrices $A_1^* = I_h \in \R^{h \times h},$ $A_2^* = (1,1,...,1)^\top \in \R^{1\times h}$ and widths $h \in \{5,6,...,29,30\}$. The number of parameters $D=h+h^2$ lies therefore in the interval $[30, 930]$. Training is done using Gaussian SGD \eqref{def:reguralized_SGD} with multivariate standard normal distributed noise variables $U_{t+1}(\mA_t)$ and hyper-parameters $\eta_t = 0.01$, $T=100$, and $m \in \{10,20,...,100\}$. 
The generalization error $R(\hat{\mA})$ is estimated by averaging over $10,000$ independently generated samples. Moreover, we compute the empirical distribution of the generalization gap $R_n(\hat{\mA}) - R(\hat{\mA})$ by repeating each simulation $200$ times.
We regard the distribution by the repetition as an empirical approximation of the distribution of outputs of the Gaussian SGD, then we estimate the probability that the Gaussian SGD stagnates in the $\delta$-neighbourhood $\mB_\delta(\mu)$ for each $\delta \in (0,0.01)$ by the empirical approximation. 
\footnote{Code is available at \url{https://github.com/insou/pop_minima}.}

Figure \ref{fig:gap_D} plots the generalization gap for different choices of $\delta$ and $D.$ The displayed generalization gap is the mean value over the repetitions. The error bars visualize the standard deviations. The plot shows empirically that the generalization gap is primarily influenced by the radius $\delta.$ Consistent with our theory, smaller values of $\delta$ lead to more implicit regularization and therefore to a decrease in the generalization gap. On the contrary, the number of parameters $D$ only plays a minor role. Interestingly, for every $\delta$ there is a regime where the generalization gap increases with the number of parameters until some saturation is attained. Beyond this regime, the number of parameters appears to have hardly any influence on the generalization gap.

Figure \ref{fig:stag_prob_nn} shows how $\delta$ and the width $h$ (recall that $h=O(\sqrt{D})$) influence the stagnation probability. In line with the theoretical bounds, the stagnation probability increases as $\delta$ get larger. The plot shows moreover, that the stagnation probability can get close to one and that a larger network width $h$ decreases the stagnation probability.

\begin{figure}
\centering
    \includegraphics[width=0.99\hsize]{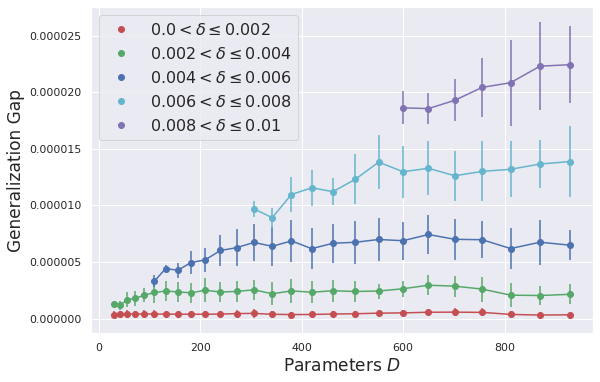}
    \caption{Dependence of the generalization gap on the number of parameters $D$ and $\delta$.}
    \label{fig:gap_D}
\end{figure}
\begin{figure}
    \centering
    \includegraphics[width=0.79\hsize]{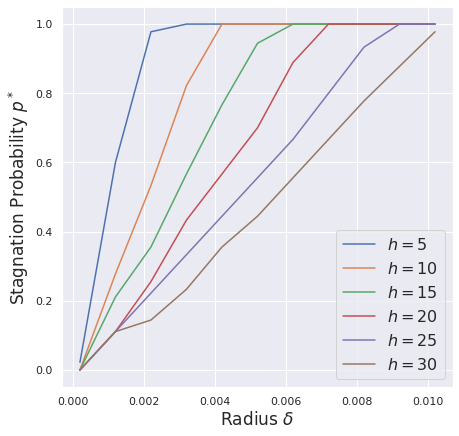}
    \caption{Stagnation probabilities in the population minima $\mB_\delta(\mu)$ for different choices of $\delta$ and $h$.
    } \label{fig:stag_prob_nn}
\end{figure}

\section{Conclusion}

This work is a first attempt to describe the generalization error of deep learning under implicit regularization induced by the local shape of the loss surface. Previous work on implicit regularization focuses on norm bounds of the weight matrices. Instead we argue that the shape of the local minima for complex loss surfaces constraint later Gaussian SGD iterations resulting in improved bounds for the generalization gap. Implicit regularization through the shape of local minima is in line with the two phases of global and local search in SGD updates, also called the \textit{two-regime} \cite{leclerc2020two}.
This is also reflected in the lower bound on the stagnation probability as it depends on the reaching~probability and the non-escaping~probability. The switch between reaching and non-escaping of local minima is driven by the decay of the learning rate. 

We view this work as a very first step to explore loss surfaced based implicit regularization and hope that the presented ideas spark further research into this direction. We want to stress that in many respects, the obtained results in this work are preliminary. While we argue that the imposed conditions are natural for deep neural networks and also occur in related work, verifying them is an extremely hard problem beyond the scope of this work. Also it is clear that much tighter lower bounds for the staying probability should be obtainable. Finally, while we obtain bounds that do not directly depend on the number of neural network parameters, it remains unclear to which extend this dependence reoccurs through other quantities in these bounds.

\section{Proof of Theorem \ref{thm:stay_prob_regularized}}

As already mentioned before, we need to lower bound the reaching probability and the non-escaping probability from $\mfB_{K,\delta}$.
Let $C_\Theta:=\sup_{\mA \in \Theta } \|\mA\|_{F}$. 
For a subset $\mB \subset \Theta$ and a parameter $\mA' \in \Theta$,  $\mB + \mA' := \{\mA + \mA' \in \Theta \mid \mA \in \mB\}$ denotes the shifted set.
We recall that if the updated parameter leaves $\Theta$, it is projected back onto $\Theta$.
However, since the proofs in this section do not depend on the location of parameters in $\Theta$ and $\mB_\delta(\mu)$ does by construction not overlap with the boundary $\Theta$, this projection does not affect the result.

\subsection{Reaching Probability}

We derive a lower bound for the reaching probability $\Pr(\mA_{\overline{t}} \in \mfB_{K,\delta})$ at time $\Bar{t}$ as displayed in \eqref{ineq:decomp_stag_prob}.

\begin{theorem}[Reaching Probability] \label{thm:reaching_prob_updated}
    Suppose Assumptions \ref{asmp:essential_convex}, \ref{asmp:local_smoothness}, and
    \ref{asmp:sample_gradient_noise} hold. 
    Consider a monotonically decreasing learning rate $\eta_t \asymp t^{-\alpha}$ with $\alpha > 0$ and let $\underline{t}$ be the first iteration $t$ such that  $t \geq C_{m, D,L,\Bar{h},G,\delta, \mfB}$ and $\eta_{t} \leq $ $\min\{C_{\mu, R, \nabla R} /({L} \Bar{h}^2 \delta^2) ,$ $ C_{\xi, \mu} n^{1/2}/( L^2 \delta^2 \Bar{h}^4 \log (\Bar{h}D))\}$.
    Suppose moreover that $n$ is sufficiently large such that  $n / \log n \geq (\eta_{\underline{t}} + 1)^{2(\beta_0 + 1)} C^2 \xi^2 d / (\beta \delta)^2$, and $\delta > C_{m,L,\Bar{h},D,G,\Theta}$ hold.
    Then, for any $\varepsilon > 0$ and any $\Bar{t}$ such that $\Bar{t} - \underline{t} \geq  C_*(- \Bar{t}^{2\alpha}C_{m,L,\Bar{h},D, \delta, G, \Theta} - \log (C_{m,D,G}\lambda (\mfB_{K,\delta})  )$,  we obtain
    \begin{align*}
        \Pr(\mA_{\bar{t}} \in \mfB_{K,\delta}) & \geq 1 - \frac {\exp\big(  - C_{m,L,\Bar{h},D, \delta, G, \Theta} \bar{t}^{2\alpha}\big)}{C_{m,D,G}\lambda (\mfB_{K,\delta})}.
    \end{align*}
\end{theorem}
\begin{proof}
Define the conditional probabilities $p_t := \Pr(\mA_{{t}} \in \mfB_{K,\delta} \mid \mA_{{t}-1} \notin \mfB_{K,\delta})$ and $q_t := \Pr(\mA_{{t}} \notin \mfB_{K,\delta} \mid \mA_{{t}-1} \in \mfB_{K,\delta}).$ 
By Lemma \ref{lem:upper_reach_prob} and Lemma \ref{lem:noise_escape}, a sufficient large $t$ such that $t \geq \max\{(C_{D,L, \Bar{h}, G} \delta^2 / \log 2), C'_{m,D,G}/ \log (2 \lambda(\mfB_{\delta,K} C_{m,G})) \}^{1/2\alpha}$ implies $p_t \leq 1/2$ and $q_t \leq 1/2$, hence we have $p_t + q_t \leq 1$.

These are the probabilities of entering or exiting the set $\mfB_{K,\delta}$. By the Markov property of the Gaussian SGD update, we have 
\begin{align}
    &\Pr(\mA_{{t}} \in \mfB_{K,\delta}) \notag \\
    &= (1-q_t)\Pr(\mA_{{t} -1} \in \mfB_{K,\delta}) + p_t \Pr(\mA_{{t} -1} \notin \mfB_{K,\delta}),  \label{eq:update_prob_in}
\end{align}
and 
\begin{align}
    &\Pr(\mA_{{t}} \notin \mfB_{K,\delta})\notag \\
    &= q_t\Pr(\mA_{{t} -1} \notin \mfB_{K,\delta}) + (1-p_t) \Pr(\mA_{{t} -1} \in \mfB_{K,\delta}). \label{eq:update_prob_out}
\end{align}
We now study the ratio
\begin{align}
    r_t :=  \frac{\Pr(\mA_{{t}} \in \mfB_{K,\delta})}{\Pr(\mA_{{t}} \notin \mfB_{K,\delta})} = \frac{\Pr(\mA_{{t}} \in \mfB_{K,\delta})}{1-\Pr(\mA_{{t}} \in \mfB_{K,\delta})} \in (0,\infty).
    \label{eq.123ratio}
\end{align}
Since the map $(0,1) \ni z \mapsto z/(1-z) \in (0,\infty)$ is bijective, there is an one-to-one correspondence between $r_t$ and $\Pr(\mA_{{t}} \in \mfB_{K,\delta})$.
Using \eqref{eq:update_prob_in} and \eqref{eq:update_prob_out}, $r_t$ can be rewritten into the recursive form
\begin{align}
    r_t &= \frac{(1-q_t) \Pr(\mA_{{t-1}} \in \mfB_{K,\delta}) + p_t (1 - \Pr(\mA_{{t-1}} \notin \mfB_{K,\delta}))}{q_t \Pr(\mA_{{t-1}} \in \mfB_{K,\delta}) + (1-p_t)(1 - \Pr(\mA_{{t-1}} \notin \mfB_{K,\delta}))} \notag \\
    &= \frac{(1-q_t) r_{t-1} + p_t}{q_t r_{t-1} + 1-p_t} =: F(r_{t-1} ; p_t,q_t). \notag 
\end{align}
Based on this recursion, we now derive a lower bound for $r_{\bar{t}}.$ By Lemma \ref{lem:r_lower_bound}, we obtain for $1 \leq \underline{t} < \Bar{t},$
\begin{align}
    r_{\bar{t}} \geq \min_{t  \in [\underline{t}, \overline{t}]} \frac{p_t}{q_t} - C (1 - \Bar{c})^{\Bar{t} - \underline{t} + 1}, \label{ineq:r_tbar}
\end{align}
for some $\Bar{c} \in (0,1)$.
Applying the bounds for $p_t$ and $q_t$ in Proposition \ref{prop:lower_reach_prob} and \eqref{ineq:lower_stay_prob}, it follows that
\begin{align}
     \frac{p_{{t}}}{q_{{t}}} &\geq \frac{ C_{m,D,G} \lambda (\mfB_{K,\delta})\exp(- C_{m,G,\Theta} {\eta}_t^{-2})}{ C_G \exp ( - C_{m,L,\Bar{h},D}\delta^2{\eta}_t^{-2})}\notag \\
     &= \exp\Big( t^{2\alpha} ( C_{m,L,\Bar{h},D} \delta^2 - C_{m,G,\Theta} )  \notag \\
     & \qquad \qquad + \log (C_{m,D,G}\lambda (\mfB_{K,\delta}))\Big), \label{ineq:ratio_lower}
\end{align}
for every $t \geq \underline{t}$. By the definition \eqref{eq.123ratio} of $r_t$, as well as \eqref{ineq:r_tbar} and \eqref{ineq:ratio_lower},
\begin{align*}
    \Pr(\mA_{\bar{t}} \in \mfB_{K,\delta}) &= \frac{r_{\Bar{t}}}{1 + r_{\Bar{t}}} \\
    & \geq 1 - \frac{1}{r_{\Bar{t}}} \\
    & \geq 1 - \left\{\exp(  \bar{t}^{2\alpha} ( C_{m,L,\Bar{h},D} \delta^2 - C_{m,G,\Theta} ) \right. \\
    & \qquad \left. + \log (C_{m,D,G}\lambda (\mfB_{K,\delta}))) - C (1- \Bar{c})^{\Bar{t} - \underline{t}}\right\}^{-1}
\end{align*}
The lower bound for $\Bar{t}-\underline{t}$, entails the inequality $2^{-1} \exp(  \bar{t}^{2\alpha} ( C_{m,L,\Bar{h},D} \delta^2 - C_{m,G,\Theta} ) + \log (C_{m,D,G}\lambda (\mfB_{K,\delta}))) \geq  C (1- \Bar{c})^{\Bar{t} - \underline{t}}$ proving
\begin{align*}
    &\Pr(\mA_{\bar{t}} \in \mfB_{K,\delta}) \\
    &\geq 1 - \frac{1}{2}\exp\Big( - \bar{t}^{2\alpha} ( C_{m,L,\Bar{h},D} \delta^2 - C_{m,G,\Theta} )  \\
    & \qquad \qquad \qquad - \log \big(C_{m,D,G}\lambda (\mfB_{K,\delta})\big)\Big).
\end{align*}
Using the condition $\delta^2 > C_{m,G,\Theta} / C_{m,L,\Bar{h},D}$ and adjusting the constants, we obtain the statement.
\end{proof}

\begin{lemma} \label{lem:r_lower_bound}
Suppose that there exists $\Bar{c} > 0$ such that $p_t, q_t \in [\Bar{c} , 1-\Bar{c}]$ and $p_t + 2 q_t \leq 1$ for all $t \in [t',T]$ with some $t' \geq 1$.
For a positive sequence $(r_t)_t$ defined as in \eqref{eq.123ratio}, we have
\begin{align*}
    r_{\bar{t}} \geq \min_{t  \in [\underline{t}, \overline{t}]} \frac{p_t}{q_t} - C (1 - \Bar{c})^{\Bar{t} - \underline{t} + 1},
\end{align*}
for some positive constant $C$.
\end{lemma}
\begin{proof}
We use that $F(\cdot; p_t, q_t)=((1-q_t)\cdot+p_t)/(q_t\cdot +1-p_t)$ is a contraction and show that $r_t$ gets close to a fixed point of $F(\cdot; p_t, q_t)$ as $t$ increases.

To prove this, we need several elementary properties of $F(\cdot; p_t, q_t).$ Since $\partial_{r'} F(r';p_t,q_t) = (1-p_t-q_t) / (1 - p_t + q_t r')^2$, $ F(r';p_t,q_t)$ is monotonically increasing in $r'$ whenever $p_t + q_t \leq 1$. Also, $F(r';p_t,q_t)$ is strictly concave in $r'$ since $\partial_{r'}^2 F(r';p_t,q_t) < 0$ holds. Moreover, we have $F(0; p_t,q_t) = p_t/(1-p_t) > 0$.

Using these properties, for each $t$, there exists a unique fixed point $r_t^*$ satisfying
\begin{align*}
    r_t^* = F(r_t^* ; p_t,q_t).
\end{align*}
Since $p_t/q_t$ is a fixed point, we must have $r_t^* = p_t/q_t$.
We show that $r_t^*$ is an attractor.
That is, for $r \in (0,\infty)$, we have
\begin{align}
    |F(r; p_t,q_t) - F(r_t^{*}; p_t,q_t)| &= \Big|1-\frac{q_t(1+r)}{1+q_t r-p_t}\Big|
     |r-r_t^*| \notag \\
     &\leq (1- \Bar{c}) |r-r_t^*|,\label{ineq:contraction_F}
\end{align}
where the last inequality follows from the fact that $q_t (1 + r) > \Bar{c}$ implies $1 - \frac{q_t(1+r)}{1+q_t r-p_t} \leq 1 - \Bar{c}$, and $q_t + (2- \Bar{c})p_t  < 2 - \Bar{c}$ implies $\frac{q_t(1+r)}{1+q_t r-p_t}  - 1\leq 1 - \Bar{c}$.
The first condition $q_t (1 + r) > \Bar{c}$ is implied by $q_t \geq \Bar{c}$, and the second condition $q_t + (2- \Bar{c})p_t  < 2 - \Bar{c}$ follows from $p_t + 2q_t \leq 1$.

Due to the concavity of $F(\cdot;p_t,q_t)$, we know that $ r \geq F(r; p_t,q_t) \geq r_t^*$ holds if and only if $r < r_{t}^*$, and $ r \leq F(r; p_t,q_t) \leq r_t^*$ holds if and only if $r > r_{t}^*$.

For any $t \in [\underline{t}, \Bar{t}]$, we derive now a lower bound for $r_{\bar{t}} - \min_{t} r_{{t}}^*$. We treat two cases, depending on whether 
the sequence $(r_t)_{t \in [\underline{t}, \overline{t}]}$ is increasing or not. If $(r_t)_{t \in [\underline{t}, \overline{t}]}$ is not an increasing sequence, there exists $t' \in [\underline{t}, \overline{t}-t]$ such that $r_{t'+1} \leq r_{t'}$. By the property of $F(\cdot; p_t, q_t)$ derived above, this implies $r_{t'} \geq r_{t'}^*$.
Because of $r_{t'}^* \geq \min_{t} r_{{t}}^*$, we know that for any $\tau \in [1, \Bar{t} - t']$ the value $r_{t' + \tau}$  cannot be larger than $\min_{t \in [\underline{t}, \overline{t}]} r_t^*$.
Thus $r_{\bar{t}} - \min_{t \in [\underline{t}, \overline{t}]} r_{{t}}^* \geq 0$.

We now consider the second case assuming that $(r_t)_{t \in [\underline{t}, \overline{t}]}$ is an increasing sequence, that is, $r_t\leq r_{t+1}$ for all $t \in [\underline{t}, \overline{t}-1]$.
Without loss of generality, we can moreover assume that $\max_{t \in [\underline{t}, \overline{t}]}r_{t} \leq \min_{t  \in [\underline{t}, \overline{t}]} r_{t}^*$ since
otherwise, by the monotone increase of $(r_t)_t$, $r_{\Bar{t}}=\max_{t \in [\underline{t}, \overline{t}]}r_{t} \geq \min_{t  \in [\underline{t}, \overline{t}]} r_{t}^* = \min_{t  \in [\underline{t}, \overline{t}]} p_t/q_t$ automatically implies the statement of Lemma \ref{lem:r_lower_bound}.
Observe that $z \mapsto (z - r_{t'}) / (z - r_{t'-1})$ is an increasing function due to $r_{t'} \geq r_{t'-1}.$ Together with \eqref{ineq:contraction_F} this gives
\begin{align*}
    \min_{t  \in [\underline{t}, \overline{t}]} r_{t}^* - r_{\bar{t}} &= \left(\min_{t  \in [\underline{t}, \overline{t}]} r_{t}^* - r_{\underline{t}}\right) \prod_{t' = \underline{t}+1}^{\Bar{t}} \frac{\min_{t  \in [\underline{t}, \overline{t}]} r_{t}^* - r_{t'}}{\min_{t  \in [\underline{t}, \overline{t}]} r_{t}^* - r_{t'-1}} \\
    & \leq \left(\min_{t  \in [\underline{t}, \overline{t}]} r_{t}^* - r_{\underline{t}}\right)  \prod_{t' = \underline{t}+1}^{\Bar{t}} \frac{ r_{t'}^* - r_{t'}}{ r_{t'}^* - r_{t'-1}} \\
    &\leq \left(\min_{t  \in [\underline{t}, \overline{t}]} r_{t}^* - r_{\underline{t}}\right) (1 - \Bar{c})^{\Bar{t} - \underline{t}}.
\end{align*}
Combining both cases, we obtain that
\begin{align*}
     r_{\bar{t}} &\geq \min_{t  \in [\underline{t}, \overline{t}]} r_{t}^* - \left(\min_{t  \in [\underline{t}, \overline{t}]} r_{t}^* - r_{\underline{t}}\right) (1 - \Bar{c})^{\Bar{t} - \underline{t}} \\
     & \geq \min_{t  \in [\underline{t}, \overline{t}]} r_{t}^* - \left(\min_{t  \in [\underline{t}, \overline{t}]} r_{t}^* \right) (1 - \Bar{c})^{\Bar{t} - \underline{t}} \\
     & \geq \min_{t  \in [\underline{t}, \overline{t}]} r_{t}^* - C_{\Bar{c}} (1 - \Bar{c})^{\Bar{t} - \underline{t} + 1},
\end{align*}
Together with the definition $r_t^*=p_t/q_t$, we obtain the result. 
\end{proof}

\begin{proposition} \label{prop:lower_reach_prob}
Consider Gaussian SGD as defined in Section \ref{sec.GSGD_intro}. For any $\mA_t \in \Theta$ and any learning rate $\eta_t \asymp t^{-2\alpha}$ with $\eta_t\leq 1$ for all $t$, we obtain
    \begin{align*}
        &\Pr( \mA_{t+1} \in \mfB_{K,\delta} \mid \mA_{t})\geq \lambda (\mfB_{K,\delta}) C_{ m, D, G} \exp \left( -C_{m,G,\Theta}t^{2\alpha} \right). 
    \end{align*}
\end{proposition}
\begin{proof}
We set $\Tilde{U}_{t+1} := m^{-1/2} {\eta_t}U_{t+1}(\mA_t) \sim \mN(0,\eta_t^2G(\mA_t)/m),$ 
and define the event $\mE_{t+1} := \{\Tilde{U}_{t+1} \in \mfB_{K,\delta} - \mA_t + \eta_t \nabla R_n(\mA_t) \}.$ By assumption, all eigenvalues of the covariance matrix $G(\mA_t)$ lie in the interval $[c_G',c_G]$ with $c_G'>0.$ Recall that $\lambda(\mB)$ denotes the Lebesgue measure of the set $\mB$ and that the Lebesgue measure is invariant under location shifts. Therefore
\begin{align*}
    &\Pr_{\Tilde{U}_{t+1}| \mA_t} ({\mE}_{t+1}) \\
    &= \int_{\mfB_{K,\delta} - \mA_t + \eta_t \nabla R_n(\mA_t)} d\Pr_{\Tilde{U}_{t+1}| \mA_t}(u) \\
    & \geq  \lambda \big(\mfB_{K,\delta} - \mA_t + \eta_t \nabla R_n(\mA_t)\big) \\
    & \quad \inf_{u \in \mfB_{K,\delta} - \mA_t + \eta_t \nabla R_n(\mA_t)} \frac{d\Pr_{\Tilde{U}_{t+1}| \mA_t}}{d\lambda}(u) \\
    &\geq 
    \lambda (\mfB_{K,\delta}) \inf_{u \in \mfB_{K,\delta} - \mA_t + \eta_t \nabla R_n(\mA_t)}
    \frac{1}{ \sqrt{ (2\pi)^D | \frac{\eta_t^2}{m}G(\mA_t)| }} \\
    & \quad \times \exp\left(- \frac{m}{2\eta_t^2 } \mathrm{vec}(u)^\top (G(\mA_t))^{-1} \mathrm{vec}(u) \right) \\
    & \gtrsim \frac{\lambda (\mfB_{K,\delta})}{ {(2\pi \eta_t^2/ m )^{D/2} | G(\mA_t)|^{1/2} }} \exp\left(- \frac{m}{2\eta_t^2c_G' } C_\Theta^2 \right)\\
    & \geq \frac{\lambda (\mfB_{K,\delta})}{ {(2\pi / m )^{D/2} | G(\mA_t)|^{1/2} }} \exp\left(- \frac{m}{2\eta_t^2 c_G'} C_\Theta^2 \right),
\end{align*}
where the second to last inequality follows from $\|u\|_2^2 \leq \max_{\mA \in \Theta} \|\mA\|_F^2 = C_{\Theta}^2$ and $\mathrm{vec}(u)^\top (G(\mA_t))^{-1} \newline \mathrm{vec}(u) \leq \|u\|_2^2 /c_G' \leq C_\Theta^2  / c_G'$ for any $\mA \in \mfB_{K,\delta}$, and in the last step we applied $\eta_t \leq 1$ for any $t$. 

Since $ct^{-\alpha}\leq \eta_{t}$, we obtain the statement 
with $C_{m, D, G} = (2\pi  c_G^2/m)^{-D/2}$ and $C_{m,G,\Theta} =m C_\Theta^2/(c c_G')^2$.
\end{proof}

\begin{lemma}\label{lem:upper_reach_prob}
    Consider Gaussian SGD as defined in Section \ref{sec.GSGD_intro}. For any $\mA_t \in \Theta$ and any learning rate $\eta_t \asymp t^{-2\alpha}$ with $\eta_t\leq 1$ for all $t$, we obtain
    \begin{align*}
        &\Pr( \mA_{t+1} \in \mfB_{K,\delta} \mid \mA_{t} \notin \mfB_{K,\delta}) \notag  \\
        &\leq  \lambda (\mfB_{K,\delta}) C'_{ m, D, G} \exp \left( -C_{m,G}t^{2\alpha} \right). 
    \end{align*}
\end{lemma}
\begin{proof}
Similar to the proof of Proposition \ref{prop:lower_reach_prob}, we set $\Tilde{U}_{t+1} := m^{-1/2} {\eta_t} U_{t+1}(\mA_t) \sim \mN(0,\eta_t^2G(\mA_t)/m),$ 
and $\mE_{t+1} := \{\Tilde{U}_{t+1} \in \mfB_{K,\delta} - \mA_t + \eta_t \nabla R_n(\mA_t) \}.$ 
Define $\Delta_t = \min_{\mA \in \mfB_{K,\delta}} \|\mA_t  - \mA\|_F > 0$.
We derive the following upper bound
\begin{align*}
    &\Pr_{\Tilde{U}_{t+1}| \mA_t \notin \mfB_{K,\delta}} ({\mE}_{t+1}) \\
    & \leq  \lambda \big(\mfB_{K,\delta} - \mA_t + \eta_t \nabla R_n(\mA_t)\big) \\
    & \quad \max_{u \in \mfB_{K,\delta} - \mA_t + \eta_t \nabla R_n(\mA_t)} \frac{d\Pr_{\Tilde{U}_{t+1}| \mA_t \notin \mfB_{K,\delta}}}{d\lambda}(u) \\
    &= 
    \lambda (\mfB_{K,\delta}) \max_{u \in \mfB_{K,\delta} - \mA_t + \eta_t \nabla R_n(\mA_t)}
    \frac{1}{ \sqrt{ (2\pi)^D | \frac{\eta_t^2}{m}G(\mA_t)| }} \\
    & \quad \times \exp\left(- \frac{m}{2\eta_t^2 } \mathrm{vec}(u)^\top (G(\mA_t))^{-1} \mathrm{vec}(u) \right) \\
    & \leq \frac{\lambda (\mfB_{K,\delta})}{ {(2\pi \eta_t^2/ m )^{D/2} | G(\mA_t)|^{1/2} }} \exp\left(- \frac{m}{2\eta_t^2c_G } \Delta_t^2 \right),
\end{align*}
where the last inequality follows from the bound on the eigenvalues of $G(\mA_t)$ and the definition of $\Delta_t$.
Since $ct^{-\alpha}\leq \eta_{t}$, we obtain the statement 
with $C'_{m, D, G} = (2\pi  (c'_G)^2/m)^{-D/2}$ and $C_{m,G} =m \Delta_t^2 /(c c_G)^2$.
\end{proof}

\subsection{Non-escaping Probability}
In this section, we derive a lower bound for $\Pr(\mA_t \in \mfB_{K,\delta}, \forall t \geq \tau \mid \tau_\mB = \tau).$
As preparation, for time $\overline{t} > \underline{t} \geq 0$, constants $c \in (0,1)$, $c' > 0$ and a parameter $\kappa > 0$, we define
\begin{align}
    \zeta(\underline{t}, \overline{t}, c, c' ,\kappa) := \prod_{\tau = \underline{t}}^{\overline{t}} (1 - c \exp( -  c' \tau^\kappa)).
    \label{eq.zeta_def}
\end{align}
Obviously, $\zeta(\underline{t}, \overline{t}, c, c' ,\kappa) \in (0,1)$ and $\zeta(\underline{t}, \overline{t}, c, c' ,\kappa)$ increases in $c$, and decreases in $c'$.

\begin{lemma}[Noise Gradient Lemma] \label{lem:noise_escape}
    Suppose the parameter space $\Theta$ contains $\{0\}$.
    Let $\mA \in \mB_\delta(\mu)$ be a parameter such that $\varepsilon = \|\mA - \Pi_\mu (\mA)\|_{L,2,1}$ with $\varepsilon \in (0,\delta)$. For an $\R^D$-valued random variable $W \sim \mathcal{N}(0,G(\mA))$ and $\delta \in (0,D)$, we obtain
    \begin{align*}
        \Pr_{W} \left( \mA + W \notin \mB_\delta(\mu) \right) \leq \exp\left(-\frac{ D(\delta - \varepsilon)^2}{128 L \bar{h} c_G' }\right).
    \end{align*}
\end{lemma}
\begin{proof}
For any $\mA' \in \mB_{\delta - \varepsilon}(\{\mA\})$, we have
$\|\mA' - \Pi_\mu(\mA)\|_{L,2,1} \leq \|\mA' - \mA\|_{L,2,1} + \| \mA - \Pi_\mu(\mA)\|_{L,2,1} \leq \delta-\varepsilon + \varepsilon = \delta$. This yields the inclusions
\begin{align}
    \mB_{\delta - \varepsilon}(\{\mA\}) \subset \mB_\delta(\Pi_\mu(\mA)) \subseteq \mB_\delta(\mu). \label{ineq:inclusion_ball}
\end{align}
We now bound the probability that $\mA + W$ does not lie in $\mB_\delta(\mu)$ anymore.
Here, we additionally consider a neighbourhood of $\{\mA\}$ in terms of the norm $\|\mA\|_F$.
Lemma \ref{lem:norm_F_L21} states that $\|\mA\|_{L,2,1} \leq \sqrt{L \bar{h}} \|\mA\|_F.$ Thus,
\begin{align*}
    &\Pr_{W} \big(\mA + W \notin \mB_{\delta - \varepsilon}(\{\mA\}) \big) \\
    &\leq \Pr_{W} \left(\mA + W \notin \left\{z: \|z - \mA\|_F \geq (\delta-\varepsilon) / \sqrt{L \bar{h}} \right\} \right).
\end{align*}
By the Gaussianity of $W$, we obtain
\begin{align*}
   &\Pr_W \left(\mA + W \notin \left\{z: \|z - \mA\|_F \geq (\delta-\varepsilon) / \sqrt{L \bar{h}} \right\} \right)\\
   &= \Pr_W \left(\{z : \|z\|_F \geq (\delta-\varepsilon) / \sqrt{L \bar{h}} \} \right) \\
    &\leq  \exp\left(-\frac{ D(\delta - \varepsilon)^2}{64 L \bar{h} c_z }\right).
\end{align*}
The last inequality follows from Proposition 1 in \cite{hsu2012tail} and since, without loss of generality, we can assume $\delta < D.$
\end{proof}

We derive now a lower bound for the non-escaping probability.

\begin{theorem}[Non-Escaping Probability] \label{thm:non-escape}
    Work under the assumptions of Theorem \ref{thm:stay_prob_regularized}.
    If $\mA_{\bar{t} }\in \mB_\delta(\mu)$ holds at time $\bar{t}$, then for any $t > \bar{t}$ such as $\eta_{t} \leq \min\{C_{\mu, R} (C_{\nabla R} \sqrt{L} \Bar{h} \delta)^{-2} ,$ $ C_{\xi, \mu} n^{1/2}/( L^2 \delta^2 \Bar{h}^4 \log (\Bar{h}D))\}$ and all sufficiently large $n$, we have
    \begin{align*}
        &\Pr_{\mA} \left( \mA_\tau \in \mB_\delta(\mu) , \forall \tau \in (\bar{t},t] \mid \mA_{\bar{t} }\in \mB_\delta(\mu)\right) \\
        &\geq \prod_{\tau = \bar{t}+ 1}^t \left\{ 1 -  C_{G} \exp \left( - C_{m,L,\Bar{h},D} \delta^2 \eta_\tau^{-2} \right)\right\}.
    \end{align*}
\end{theorem}
\begin{proof}
For the $t$-th iteration, we define the following event
\begin{align*}
    \mE_t := \{\mA_t \in \mB_\delta(\mu)\}.
\end{align*}
By the Markov property of the Gaussian SGD update rule \eqref{def:simple_SGD}, we have
\begin{align}
    &\Pr_{\mA} \big( \mA_\tau \in \mB_\delta(\mu) , \forall \tau \in (\bar{t},t] \mid \mA_{\bar{t} }\in \mB_\delta(\mu)\big) \notag  \\
    &= \Pr_\mA \left( \bigcap_{\tau = \bar{\tau}+1}^t \mE_\tau  \mid \mE_{\bar{t}}\right) = \prod_{\tau = \bar{t} + 1}^t \Pr_\mA \left(  \mE_\tau \mid \mE_{\tau - 1}\right). \label{ineq:prob_non-escape1}
\end{align}
We now evaluate the step-wise non-escaping probability $\Pr_\mA \left(  \mE_\tau \mid \mE_{\tau - 1}\right)$ for $\tau \geq \bar{t} + 1$. In a first step, we consider the parameter update $\mA_{\tau}^- := \mA_{\tau-1} - \eta_{\tau-1} \nabla R_n(\mA_{\tau - 1})$ and study the probability that $\mA_{\tau}^- \in \mB_\delta(\mu)$ holds. To this end, we consider the obvious inequality $\|\mA_{\tau}^- - \Pi_\mu(\mA_{\tau}^-)\|_{L,2,1} \leq \|\mA_{\tau}^- - \Pi_\mu(\mA_{\tau-1})\|_{L,2,1}$. 
By Lemma \ref{lem:local_attraction}, we obtain
\begin{align*}
    \|\mA_{\tau}^- - \Pi_\mu(\mA_{\tau-1})\|_{L,2,1} < \|\mA_{\tau- 1} - \Pi_\mu(\mA_{\tau-1})\|_{L,2,1},
\end{align*}
with probability at least $1-\varepsilon$ for the specified $\eta_{\tau - 1}$ and $n$.
Setting $\varepsilon = \exp(-1/\eta_{\tau-1}^2)$ gives
\begin{align*}
    \|\mA_{\tau}^- - \Pi_\mu(\mA_{\tau-1}^-)\|_{L,2,1}& \leq \|\mA_{\tau-1}^- - \Pi_\mu(\mA_{\tau-1})\|_{L,2,1} \\
    &\leq \|\mA_{\tau-1} - \Pi_\mu(\mA_{\tau-1})\|_{L,2,1}.
\end{align*}
Hence, we have $\mA_{\tau}^- \in \mB_\delta(\mu)$.

We apply Lemma \ref{lem:noise_escape} and evaluate $\Pr_\mA \left(  \mE_\tau \mid \mE_{\tau - 1}\right)$. Replacing in the notation of Lemma \ref{lem:noise_escape}
the parameter $\mA$ by $\mA_{\tau - 1}$ and the gradient noise $U$ by $\frac{\eta_\tau}{ \sqrt{m} }U_{\tau}(\mA_{\tau-1})$, we obtain
\begin{align*}
    &\Pr_\mA \left(  \mE_\tau \mid \mE_{\tau - 1}\right) \\
    &= \Pr_{U} \left( \mA + U \in \mB_\delta(\mu)\right)\\
    & = 1 - \Pr_U(\mA + U \notin \mB_\delta(\mu)) - \varepsilon \\
    & \geq 1 - \exp\left(-\frac{ C Dm\delta^2}{L \bar{h} \eta_{\tau}^2 c_G'}\right) \sqrt{\frac{|m^{-1}\eta_\tau^2 I|}{|\Sigma_U|}} - \exp \left(-\frac{1}{\eta_\tau^{2}} \right) ,
\end{align*}
where $\Sigma_U$ is the covariance matrix of $\frac{\eta_\tau}{ \sqrt{m} }U_{\tau}(\mA_{\tau-1})$, and $C$ is a universal constant. 
From that, we obtain 
\begin{align}
    \Pr_\mA \left(  \mE_\tau \mid \mE_{\tau - 1}\right) \geq 1 -  C_{G} \exp \left( - C_{m,L,\Bar{h},D} \delta^2 \eta_\tau^{-2} \right), \label{ineq:lower_stay_prob}
\end{align}
with $C_{G} = 2 \sqrt{\frac{|m^{-1} \eta_{\tau}^2I|}{|\Sigma_U|}} = 2 \sqrt{\frac{|I|}{|G(\mA_{\tau - 1})|}}$ and $C_{m,L,\Bar{h},D} = CDm/( L \Bar{h})$.
We substitute this result into \eqref{ineq:prob_non-escape1} and obtain the statement.
\end{proof}

Next, we derive a local attraction property of the parameter updates. As preparation, we define for each positive integer $n$ and any $\varepsilon>0,$
\begin{align}
    \Xi_{n,\varepsilon} = C_{\xi,\mu} n^{-1/4} \sqrt{{ L^2 \delta^2 \Bar{h}^4  \varepsilon^{-1} \log (\Bar{h}D/\varepsilon) }}. \label{def:xi}
\end{align}
Lemma \ref{lem:gradient_convergence_entropy} shows that with probability at least $1-\varepsilon,$ this expression is an upper bound for $\|R_n (\mA) - R(\mA)\|$ and all  $\mA\in \mB_\delta(\mu)$.

\begin{lemma}[Local Attraction] \label{lem:local_attraction}
Suppose Assumptions \ref{asmp:essential_convex} and \ref{asmp:local_smoothness} hold.
Let $\varepsilon > 0,$ $\eta \leq  C_{\mu, R}(C_{\nabla R} L^{1/2} \Bar{h} \delta )^{-2}/2$ and assume that $n$ is sufficiently large. Then, with probability $1-\varepsilon,$ any $\mA \in \mB_\delta(\mu) \backslash \mu$ satisfies
\begin{align*}
    &\|\mA - \eta \nabla R(\mA) - \Pi_\mu(\mA)\|_{L,2,1} <  \|\mA - \Pi_\mu(\mA)\|_{L,2,1}.
\end{align*}
\end{lemma}
\begin{proof}
Here, $A_{\ell,j} = (A_{\ell,j,1},...,A_{\ell,j,h_{\ell - 1}})^\top \in \R^{h_{\ell}}$ denotes the $j$-th row of $A_{\ell}$, and $\nabla_{A_{\ell,j}}$ denotes the derivative with respect to $A_{\ell,j}$.
The difference of the two norms in the statement can be rewritten as
\begin{align}
    & \|\mA - \eta \nabla R(\mA) - \Pi_\mu(\mA)\|_{L,2,1} - \|\mA - \Pi_\mu(\mA)\|_{L,2,1} \notag \\
    &= \sum_{\ell=1}^L \sum_{j=1}^{h_{\ell}}\Delta_{\ell,j} \label{def:delta},
\end{align}
where we define 
\begin{align*}
    \Delta_{\ell,j} &:= \|A_{\ell, j} - \eta \nabla_{A_{\ell, j}} R_n(\mA) - (\Pi_\mu(\mA))_{\ell, j}\|_2 \\
    & \quad - \|A_{\ell, j} - (\Pi_\mu(\mA))_{\ell, j}\|_2. 
\end{align*}
To study $\Delta_{\ell,j}$, we use
\begin{align*}
    &\|A_{\ell, j} - \eta \nabla_{A_{\ell, j}} R_n(\mA) - (\Pi_\mu(\mA))_{\ell, j}\|_2^2 \\
    &=  \|A_{\ell, j} - (\Pi_\mu(\mA))_{\ell, j}\|_2^2 +  \Tilde{\Delta}_{\ell, j},
\end{align*}
where we define
\begin{align*}
    \Tilde{\Delta}_{\ell, j} &=   \eta^2 \|\nabla_{A_{\ell, j}} R_n(\mA)\|_2^2 \\
    & \quad - \eta \langle \nabla_{A_{\ell, j}} R_n(\mA), A_{\ell, j} - (\Pi_\mu(\mA))_{\ell, j} \rangle.
\end{align*}
By sorting the terms and using that for $a,b>0,$ $a-b=(a^2-b^2)/(a+b)$, we define
\begin{align*}
    \Check{ \Delta}_{\ell,j} &:= \|A_{\ell, j} - (\Pi_\mu(\mA))_{\ell, j}\|_2 \\
    & \quad + \|A_{\ell, j} - \eta \nabla_{A_{\ell, j}} R_n(\mA) - (\Pi_\mu(\mA))_{\ell, j}\|_2,
\end{align*}
and obtain
\begin{align*}
    \Delta_{\ell,j} =\frac{\Tilde{\Delta}_{\ell,j}}{ \Check{ \Delta}_{\ell,j}}.
\end{align*}
The denominator $\Check{ \Delta}_{\ell,j}$ is positive.
Substituting the previous display into \eqref{def:delta} and applying H\"older's inequality, we obtain
\begin{align*}
    &\|\mA - \eta \nabla R(\mA) - \Pi_\mu(\mA)\|_{L,2,1} - \|\mA - \Pi_\mu(\mA)\|_{L,2,1}  \\
    &\leq \left(\sum_{\ell=1}^L \sum_{j=1}^{h_{\ell}}\Tilde{\Delta}_{\ell,j} \right) \left( \max_{\ell= 1,...,L} \max_{j=1,...,h_\ell} \Check{ \Delta}_{\ell,j}^{-1} \right).
\end{align*}
To prove the result, it remains to show that $\sum_{\ell=1}^L \sum_{j=1}^{h_{\ell}}\Tilde{\Delta}_{\ell,j}$ is negative. Lemma \ref{lem:norm_F_L21} and Lemma \ref{lem:gradient_lip} ensure that
\begin{align}
    \|\nabla R(\mA)\|_F \notag &= \|\nabla R(\mA) - \nabla R(\Pi_{\mu}(\mA))\|_F \notag \\
    &\leq C_{\nabla R} \sqrt{L} \Bar{h} \| R(\mA) - R(\Pi_{\mu}(\mA))\|_{F}\notag \\
    &\leq C_{\nabla R} \sqrt{L} \Bar{h} \| R(\mA) - R(\Pi_{\mu}(\mA))\|_{L,2,1} \notag \\
    &\leq C_{\nabla R} \sqrt{L} \Bar{h} \delta. \label{ineq:bound_nabla_R}
\end{align}
By the essential convexity condition in Assumption \ref{asmp:essential_convex}, both $\nabla R(\mA)$ and $\mA - \Pi_\mu (\mA)$ are descent direction. Therefore, there exists $C_{\mu, R} > 0$, such that
\begin{align}
    \langle \nabla R(\mA), \mA - \Pi_\mu(\mA) \rangle &\geq C_{\mu, R} \notag \\
    &\geq \frac{C_{\mu, R}}{2}+ \frac{C_{\mu, R} \|\nabla R(\mA)\|_F^2}{2(C_{\nabla R} \sqrt{L} \Bar{h} \delta)^2}. \label{ineq:descent_direction}
\end{align}
With ${\Xi}_{n,\varepsilon}$ as in \eqref{def:xi}, the probabilistic bound  derived in Lemma \ref{lem:gradient_convergence_entropy} gives for any $\varepsilon \in (0,1),$
\begin{align*}
    \sum_{\ell=1}^L \sum_{j=1}^{h_{\ell}}\Tilde{\Delta}_{\ell,j}&= \sum_{\ell=1}^L \sum_{j=1}^{h_{\ell}}\eta^2 \|\nabla_{A_{\ell, j}} R_n(\mA)\|_2^2 \\
    & \quad - \eta \langle \nabla_{A_{\ell, j}} R_n(\mA), A_{\ell, j} - (\Pi_\mu(\mA))_{\ell, j} \rangle \\
    & = \eta^2 \|\nabla R_n(\mA)\|_F^2 - \eta \langle \nabla R_n(\mA), \mA - \Pi_\mu(\mA) \rangle \\
    & = \eta^2 \|\nabla R(\mA)\|_F^2 - \eta \langle \nabla R(\mA), \mA - \Pi_\mu(\mA) \rangle \\
    & \quad + \eta^2 (\|\nabla R_n(\mA)\|_F^2 - \|\nabla R(\mA)\|_F^2 )  \\
    & \quad + \eta \langle \nabla R(\mA) - \nabla R_n(\mA), \mA - \Pi_\mu(\mA) \rangle\\
    & \leq  \bigg(\eta^2 -  \frac{\eta C_{\mu, R}}{2(C_{\nabla R} \sqrt{L} \Bar{h} \delta)^2}  \bigg)\|\nabla R(\mA)\|_F^2 - \eta\frac{C_{\mu, R}}2 \\
    & \quad + \eta^2 \big(\| \nabla R(\mA) - \nabla R_n(\mA)\|_F^2\\
    & \quad + 2 \| \nabla R(\mA) - \nabla R_n(\mA)\|_F \| \nabla R(\mA)\|_F\big) \\
    & \quad + \eta \| \nabla R(\mA) - \nabla R_n(\mA)\|_F\|\mA - \Pi_\mu(\mA)\|_F \\
    & \leq  \bigg(\eta^2 -  \frac{\eta C_{\mu, R}}{2(C_{\nabla R} \sqrt{L} \Bar{h} \delta)^2}  \bigg)\|\nabla R(\mA)\|_F^2- \eta\frac{C_{\mu, R}}2 \\
    & \quad + \eta^2 ({\Xi}_{n,\varepsilon}^2 + 2 {\Xi}_{n,\varepsilon} C_{\nabla R}\sqrt{L}\Bar{h} )  + \eta {\Xi}_{n,\varepsilon} \delta
\end{align*}
with probability at least $1-\varepsilon$. Here, the first inequality follows from \eqref{ineq:descent_direction}, the elementary inequality $\|a\|_2^2 - \|b\|_2^2  = \|a-b\|_2^2 + 2 \langle a-b,b\rangle \leq \|a-b\|_2^2 + 2\|a-b\|_2 \|b\|_2$ and Cauchy-Schwartz inequality.
The second inequality combines the entropy bound in Lemma \ref{lem:covering_neighbour} and Theorem 1 in \cite{mei2018landscape}.
Furthermore, it follows from the property of minima that $\nabla_{A_{\ell, j}} R(\mA)\mid_{\mA = \Pi_\mu(\mA)} = 0$.
Based on this inequality, $\eta \leq   C_{\mu, R} (C_{\nabla R} L^{1/2} \Bar{h} \delta )^{-2}/2$ and the fact that ${\Xi}_{n,\varepsilon}\to 0$ as $n\to \infty,$ we obtain
\begin{align*}
     \sum_{\ell=1}^L \sum_{j=1}^{h_{\ell}}\Tilde{\Delta}_{\ell,j} < 0
\end{align*}
for all sufficiently large $n.$ The proof is complete.
\end{proof}

\subsection{Combining all bounds}

\begin{proof}[Proof of Theorem \ref{thm:stay_prob_regularized}]
Arguing as for the inequality in \eqref{ineq:decomp_stag_prob}, we obtain
\begin{align*}
    &\Pr(\mA_t \in \mfB_{K,\delta}, \forall t \in  [\overline{t}, T]) \\
    &=\Pr(\mA_t \in \mfB_{K,\delta}, \forall t \in  [\overline{t} +1 , T] \mid \mA_{\overline{t}} \in \mfB_{K,\delta}) \Pr(\mA_{\overline{t}} \in \mfB_{K,\delta}) \\
    & \geq  \Big(1 - \frac{\exp(  - C_{m,L,\Bar{h},D, \delta, G, \Theta} \bar{t}^{2\alpha} )}{C_{m,D,G}\lambda (\mfB_{K,\delta})}\Big) \\
    & \quad \times \prod_{\tau = \bar{t}+ 1}^T \left\{ 1 -  C_{G} \exp \left( - C_{m,L,\Bar{h},D} \delta^2 \eta_\tau^{-2} \right)\right\},
\end{align*}
by Theorem \ref{thm:reaching_prob_updated} and Theorem \ref{thm:non-escape}. Bernoulli's inequality yields 
\begin{align*}
    &\prod_{\tau = \bar{t}+ 1}^T \left\{ 1 -  C_{G} \exp \left( - C_{m,L,\Bar{h},D} \delta^2 \eta_\tau^{-2} \right)\right\} \\
    &\geq \Big(1- C_{G} \exp\big(- C_{m,L,\Bar{h},D} \delta^2 (\overline{t} + 1)^{2\alpha}\big)\Big)^{T -\overline{t } - 1} \\
    & \geq 1 - (T -\overline{t } - 1) C_{G} \exp\Big(- C_{m,L,\Bar{h},D} \delta^2 (\overline{t} + 1)^{2\alpha}\Big) \\
    &= 1 - \exp\Big(- C_{m,L,\Bar{h},D} \delta^2 (\overline{t} + 1)^{2\alpha} + \log (T-\overline{t}-1)\Big).
\end{align*}
The right hand side can be lower bounded by an expression of the form $\geq 1 -  \exp( - C_{T,\Bar{t},m,L, \Bar{h}, D, \alpha} \delta^2)$ for all sufficiently large $\delta$.
The assertion follows.
\end{proof}

\section{Entropy Bound}
We derive a bound for the covering number of the ball around a local minimum $\mB_\delta(\mu)$. Interestingly, the covering number does not depend on the number of network parameters $D.$

\begin{lemma}[Lemma 3.2 in \cite{bartlett2017spectrally}] \label{lem:sparsification}
    Let $A \in \R^{d \times m}$ and $X \in \R^{n \times d}$.
    If $p,q,r,s \in [1,\infty]$ satisfy $1/p+1/q=1, 1/s + 1/r=1,$ and $p \leq 2$, then,
    \begin{align*}
        &\log \mN\big(\varepsilon, \left\{XA: \|A\|_{q,s} \leq a, \|X\|_{F}\leq b \right\}, \|\cdot\|_F \big)\\
        &\leq \left\lceil \frac{a^2 b^2 m^{2/r}}{ \varepsilon^2} \right\rceil \log (2dm).
    \end{align*}
\end{lemma}
Applying this lemma with $p=q=2, r= \infty,s=1,n=d$, and $X=I_d$ gives
    \begin{align}
        \log \mN\big( \varepsilon, \{A \in \R^{d \times h}, \|A\|_{2,1} \leq a\}, \|\cdot\|_F\big)\leq \frac{a^2 d^2}{\varepsilon^2} \log (2dh).
    \label{lem:sparsification_2}
    \end{align}
This result shows that for covering a $\|\cdot\|_{2,1}$-ball, the number of columns $d$ is important, but not the overall number of matrix elements. \cite{bartlett2017spectrally} uses the previous bound to control the covering number of neural networks.
The following results apply the covering bound to measure the size of the neighborhood $\mB_\delta(\mu)$.

\subsection{Entropy Bound Results}

\begin{lemma}[Covering Number of $\mB_\delta(\mu)$]\label{lem:covering_neighbour_point}
    For $\mA \in \Theta$ and $\varepsilon > 0$, we have
    \begin{align*}
        \log \mN \big(\varepsilon, \mB_\delta(\{\mA\}), \|\cdot\|_F\big) \leq  \frac{ L \delta^2 \Bar{h}^2}{\varepsilon^2} \log (2 \Bar{h}).
    \end{align*}
\end{lemma}
\begin{proof} Given $\mA = (A_1,...,A_L),$ define for each $\ell=1,\ldots,L,$ and any $\varepsilon > 0$,
\begin{align*}
    \mB_{\varepsilon}^\ell (A_\ell) = \left\{A_\ell ' \in \R^{h_{\ell-1} \times h_\ell} \, \big |\, \|A_{\ell} ' - A_\ell\|_{2,1} \leq \varepsilon \right\}.
\end{align*}
Since a location shift does not change the covering number, \eqref{lem:sparsification_2} gives
\begin{align*}
    \log \mN( \tau,  \mB_{\delta}^\ell (A_\ell), \|\cdot\|_F ) \leq \frac{\delta^2 \Bar{h}^2}{\tau^2} \log (2 \Bar{h}) =: \log (N_\ell (\tau)),
\end{align*}
for $\tau > 0$. For each $\ell$, we now consider a covering set of $\mB_{\delta}^\ell(A_\ell)$ denoted by $\{A_{\ell,j}\}_{j=1}^{N_\ell(\tau)}$. We argue that $\{(A_{1,j_1},...,A_{L,j_L})\}_{j_1,...,j_L = 1}^{N_1(\tau),...,N_L(\tau)}$ are the centers of a $\sqrt{L} \tau$-covering of $\mB_\delta(\{\mA\}).$ Indeed, for any $\mA' \in \mB_\delta(\{\mA\})$, there exists $j_1,...,j_L$ such that
\begin{align*}
    \|\mA' -  (A_{1,j_1},...,A_{L,j_L})\|_F^2 = \sum_{\ell = 1}^L \|A_\ell ' - A_{\ell,j_\ell}\|_F^2 \leq L \tau^2.
\end{align*}
Since the cardinality of $\{(A_{1,j_1},...,A_{L,j_L})\}_{j_1,...,j_L = 1}^{N_1(\tau),...,N_L(\tau)}$ is bounded by $\prod_{\ell = 1}^L N_\ell (\tau)$, taking $\tau=\varepsilon/\sqrt{L}$ yields the claim.
\end{proof}

\begin{lemma}\label{lem:covering_neighbour}
    For any $\varepsilon \in (0, (L \Bar{h})^{1/2}]$, we have
    \begin{align*}
        \log \mN (\varepsilon, \mfB_{K,\delta}, \|\cdot\|_F) \leq C_\mu d \log \Big(\frac K{\varepsilon}\Big) + \frac{ 4 L \delta^2 \Bar{h}^2}{\varepsilon^2} \log (2 \Bar{h}).
    \end{align*}
\end{lemma}
\begin{proof} It is enough to consider the case $K=1.$ In this case, $\mfB_{K,\delta}=\mB_\delta(\mu).$ We now show that
\begin{align}
    &\log \mN (\varepsilon, \mB_\delta(\mu), \|\cdot\|_F)  \notag  \\
    &\leq \log \mN(\varepsilon, \mu, \|\cdot\|_F) + \sup_{\mA \in \mu} \log \mN(\varepsilon, \mB_{2\delta}(\{\mA\}), \|\cdot\|_F). \notag 
\end{align}
Note that the second term depends on a $2\delta$-neighbourhood of $\mA$.
For $\varepsilon > 0$, we set $N_1 = N_1(\varepsilon):=\mN(\varepsilon, \mu, \|\cdot\|_F)$ and $N_2 = N_2(\varepsilon):=\sup_{\mA \in \mu}\mN(\varepsilon, \mB_{2\delta}(\{\mA\}), \|\cdot\|_F).$ 
Let $\{\mA_j\}_{j=1}^{N_1(\varepsilon_1)}$ denote the centers of the $\varepsilon$-balls covering $\mu.$
Also, for each $\mA_j$, let $\{\mA_{j,k}\}_{k=1}^{N_2(\varepsilon_2)}$ contain the centers of the $\varepsilon$-balls covering $\mB_{2\delta}(\mA_j).$

Let $\mA \in \mB_\delta(\mu)$ be arbitrary.
We denote its projection on $\mu$ by $\Pi_\mu(\mA)$, and write $\mA_j$ for its nearest element in the covering set of $\mu,$ that is, $j \in \argmin_{j' = 1,...,N_1} \|\mA_{j'} - \Pi_\mu(\mA)\|_{F}$.
For $\varepsilon \in (0, \delta/\sqrt{L\Bar{h}} ]$, we obtain
\begin{align*}
    \|\mA - \mA_j\|_{L,2,1}
    &\leq \|\mA - \Pi_\mu(\mA)\|_{L,2,1} + \|\Pi_\mu(\mA) - \mA_j\|_{L,2,1} \\
    & \leq \delta + \sqrt{L\Bar{h}}  \|\Pi_\mu(\mA) - \mA_j\|_F \leq \delta + \sqrt{L \Bar{h} } \varepsilon \leq 2 \delta,
\end{align*}
by the fact $\mA \in \mB_{2\delta}(\{\mA_j\})$ and Lemma \ref{lem:norm_F_L21}.
Then, there exists $\mA_{j,k}$ from the covering set of $\mB_{2\delta}$ such that $\|\mA - \mA_{j,k}\| \leq \varepsilon$. Together with Lemma \ref{lem:covering_neighbour_point}, we obtain the result.
\end{proof}

We now study the set of neural networks functions with parameters in $\mB_\delta(\mu)$. More precisely, to bound the Rademacher complexity it is natural to work with vectors of function values.
For the set of  $n$ inputs $X_{1:n} = (X_1,...,X_n)$, define
\begin{align*}
    \mF(\mB_{\delta}(\mu)\mid X_{1:n}) := \{(f(X_1),...,f(X_n)): f \in \mF(\mB_{\delta}(\mu))\}
\end{align*}
and recall that $\|\cdot\|_s$ denotes the spectral norm.
\begin{proposition} \label{prop:covering}
    Assume that the activation function is $1$-Lipschitz. Let $s_\ell := \sup_{\mA \in \mB_\delta(\mu)} \|A_\ell\|_{s} $ and $S := \prod_{\ell=1}^L s_\ell$.
    Then, we have
        \begin{align*}
        &\log\mN(\varepsilon, \mF(\mfB_{K,\delta}\mid X_{1:n}),\|\cdot\|_F ) \\
        &\leq C_\mu d \log\Big(\frac {K}{\varepsilon}\Big)  + \frac{4 n h_0 S^2 L^2 \delta^2 \log (2 \Bar{h}^2) }{\varepsilon^2}.
    \end{align*}
\end{proposition}
\begin{proof} Introduce the reference parameter $\mM = (M_1,...,M_L) \in \R^D$ and denote the input data matrix by $\mathbf{X} = (X_1,...,X_n) \in \R^{h_0 \times n}.$ Define now the neural network function with matrix input as
\begin{align*}
    F_\mA(\mathbf{X}) &= A_L \sigma A_{L-1} \cdots A_2 \sigma A_1 \mathbf{X} \\
    &= \big(F_\mA(X_1),...,F_\mA(X_n)\big) \in \R^{h_L \times n}.
\end{align*}
and, for $k_1,...,k_L \geq 0$, consider the following hypothesis set
\begin{align*}
    &\mG\big(k_1,...,k_L\mid \mM, X_{1:n}\big) \\
    &:= \left\{ F_\mA(\mathbf{X}) \mid \|A_j - M_j\|_{2,1} \leq k_j, j = 1,...,L \right\}.
\end{align*}
We also define $\Bar{k} = \max_\ell k_\ell$ and $s_\ell := \sup_{A_\ell: \|A_j - M_j\|_{2,1} \leq k_j} \|A_\ell\|_{s}$.

Firstly, we derive a bound for the covering number of $\mG(k_1,...,k_L\mid \mM, X_{1:n})$ for given $\mM,$ by showing that for all sufficiently small $\varepsilon > 0$,
\begin{align}
    &\log \mN\big(\varepsilon, \mG(k_1,...,k_L\mid \mM, X_{1:n}), \|\cdot\|_F\big) \notag \\
    &\leq n h_0 \frac{L^3\bar{k}^2  S^2 }{ \varepsilon^2} \log (2\Bar{h}^2). \label{ineq:cover_ineq1}
\end{align}
We define the sub-tuple of parameter matrices as $\mA_{1:\ell} := (A_1,...,A_\ell)$ and $F_{A_{1:\ell}}$ denotes the network consisting of the first $\ell$ layers of $F_\mA$.

To show the covering bound, we define
\begin{align*}
    \mF_\ell(M_\ell,k_\ell) := \left\{A_\ell F_{\mA_{1:\ell-1}}(\mathbf{X}),  \|A_\ell- M_\ell\|_{2,1} \leq k_\ell \right\}
\end{align*}
and show by induction that
\begin{align}
    &\log \mN(\varepsilon, \mG(k_1,...,k_L\mid \mM, X_{1:n}), \|\cdot\|_F) \notag \\
    &\leq \sum_{\ell=1}^L \sup_{\mA_{1:\ell-1}} \log \mN\left(\varepsilon_\ell,  \mF_\ell(M_\ell,k_\ell), \|\cdot\|_F \right) \label{ineq:covering1},
\end{align}
with values $\varepsilon_1,...,\varepsilon_L > 0$ that depend on $\varepsilon$.
To derive \eqref{ineq:covering1}, we consider for each $\ell=1,...,L-1$ a covering $\{A_{\ell,j}\}_{j=1}^{N_\ell}$ of the set $\{A_\ell F_{\mA_{1:\ell-1}}(\mathbf{X}) \mid  \|A_\ell- M_\ell\|_{2,1} \leq k_\ell \}$ with corresponding covering number
\begin{align*}
    N_\ell = \sup_{\mA_{1:\ell-1}: \|A_j - M_j\|_{2,1} \leq k_j} \mN(\varepsilon_\ell , \mF_\ell(M_\ell,k_\ell), \|\cdot\|_F).
\end{align*}
Now, we show that the vectors
\begin{align}
    A_{L,j_L} \sigma(A_{L-1,j_{L-1}} \cdots  \sigma(A_{1,j_1} \textbf{X}) \cdots ), \label{ineq:compose}
\end{align}
for all combinations of $j_\ell = 1,...,N_\ell$ and all $\ell = 1,...,L$ are the centers of a covering for $\mG(k_1,...,k_L \mid  \mM, X_{1:n})$. Observe that the number of elements in the covering is $\prod_{\ell = 1}^L N_\ell$.

To find the radius of the covering,
fix arbitrary $(A_1,...,A_L).$ Then, there exists a tuple $(j_1,...,j_L)$ such that by applying the triangle inequality iteratively for each layer and the $1$-Lipschitz property of the activation function $\sigma$,
\begin{align}
    &\Big\|A_{L} \sigma(A_{L-1} \cdots  \sigma(A_{1} \textbf{X}) \cdots ) \notag \\
    & \qquad - A_{L,j_L} \sigma(A_{L-1,j_{L-1}} \cdots  \sigma(A_{1,j_1} \textbf{X}) \cdots ) \Big\|_F \notag \\
    &\leq \Big\|A_{L} \sigma(A_{L-1} \cdots  \sigma(A_{1} \textbf{X}) \cdots ) \notag \\
    & \qquad - A_{L} \sigma(A_{L-1,j_{L-1}} \cdots  \sigma(A_{1,j_1} \textbf{X}) \cdots ) \Big\|_F \notag \\
    &\quad + \Big\|A_{L,j_L} \sigma(A_{L-1} \cdots  \sigma(A_{1} \textbf{X}) \cdots )  \notag \\
    & \qquad \quad  - A_{L,j_L} \sigma(A_{L-1} \cdots  \sigma(A_{1} \textbf{X}) \cdots ) \Big\|_F \notag \\
    &\leq \|A_L\|_{s} \Big\|\sigma(A_{L-1} \cdots  \sigma(A_{1} \textbf{X}) \cdots ) \notag  \\
    & \qquad - \sigma(A_{L-1,j_{L-1}} \cdots  \sigma(A_{1,j_1} \textbf{X}) \cdots ) \Big\|_F  +  \varepsilon_L \notag \\
    & \leq s_L \Big\|A_{L-1} \sigma(A_{L-2} \cdots  \sigma(A_{1} \textbf{X}) \cdots )  \notag \\
    & \qquad - A_{L-1,j_L} \sigma(A_{L-2,j_{L-1}} \cdots  \sigma(A_{1,j_1} \textbf{X}) \cdots ) \Big\|_F + \varepsilon_L \notag \\
    & \leq \varepsilon_L + s_L \varepsilon_{L-1} + s_L s_{L-1} \varepsilon_{L-2} + \cdots + (s_L s_{L-1} \cdots s_1) \varepsilon_1 \notag \\
    & \leq  \sum_{\ell = 1}^L  \varepsilon_\ell  \left(\prod_{\ell' = \ell + 1}^{L} s_{\ell'} \right), \notag 
\end{align}
where we set $\prod_{\ell' = L + 1}^{L} s_{\ell'} = 1$ in the final upper bound. From this inequality, we obtain that \eqref{ineq:compose} forms a $\sum_{\ell = 1}^L  \varepsilon_\ell (\prod_{\ell' = \ell + 1}^{L} s_{\ell'} )$-covering. Set $c_\ell := 1/(L \Pi_{\ell' = \ell + 1}^{L} s_{\ell'})$ and $\varepsilon_\ell := c_\ell\varepsilon.$ Using \eqref{ineq:covering1} and Lemma \ref{lem:sparsification}, we obtain 
\begin{align*}
    &\log \mN(\varepsilon, \mG(k_1,...,k_L\mid  \mM, X_{1:n}), \|\cdot\|_F)\\
    &\leq \sum_{\ell=1}^L \sup_{\mA_{1:\ell-1}} \log \mN\left(c_\ell \varepsilon,  \mF_\ell(M_\ell,k_\ell), \|\cdot\|_F \right)\\
    &\leq \sum_{\ell=1}^L \sup_{\mA_{1:\ell-1}} \frac{k_\ell^2 \|F_{\mA_{1:\ell-1}}(\textbf{X}) \|_F^2 }{(c_\ell \varepsilon)^2} \log (2\Bar{h}^2).
\end{align*}
For any $j=1,...,L$, the norm $\| F_{\mA_{1:\ell-1}}(\textbf{X})\|_s$ can be bounded iteratively by
\begin{align*}
    &\| F_{\mA_{1:\ell-1}}(\textbf{X})\|_s \leq \| F_{\mA_{1:\ell-1}}(\textbf{X})\|_F \notag \\
    &\leq \|A_{\ell-1}\|_s \|F_{\mA_{1:\ell-2}}(\textbf{X})\|_F\leq \cdots \leq \|\textbf{X}\|_F \prod_{\ell'=1}^{\ell-1} s_{\ell'},
\end{align*}
using again the $1$-Lipschitz continuity of $\sigma$.
Combining the bound with the definition of $c_\ell$, we finally obtain that
\begin{align}
    &\log \mN(\varepsilon, \mG(k_1,...,k_L\mid  \mM, X_{1:n}), \|\cdot\|_F) \notag \\
    &\leq \sum_{\ell=1}^L \frac{k_\ell^2 \|\textbf{X}\|_F^2 \prod_{\ell'=1}^{\ell-1} s_{\ell'}^2 }{(c_\ell \varepsilon)^2} \log (2\Bar{h}^2). \label{ineq:bound_f_covering}
\end{align}
Together with $\|\textbf{X}\|_F = ({ \sum_{i=1}^n \|X_i\|_2^2})^{1/2} \leq \sqrt{n} \sqrt{h_0}$, we obtain \eqref{ineq:cover_ineq1}. 

Finally, we derive a covering bound for $\mN(\varepsilon, \mF(\mfB_{K,\delta} \mid X_{1:n}), \|\cdot\|_F)$, by following the proof of Lemma \ref{lem:covering_neighbour}. We study the case $K=1$.
For $\varepsilon' > 0$, we set $N'_1 = N'_1(\varepsilon'):=\mN(\varepsilon_1', \mu, \|\cdot\|_F)$ and write $\{\mA_j\}_{j=1}^{N_1'}$ for the $\varepsilon'$-covering set for $\mu$.
Also, for $j=1,...,N'_1$, we set  $N'_2 = N'_2(\varepsilon'):=\mN(\varepsilon', \mG(k_1,...,k_L \mid \mA_j, X_{1:n}), \|\cdot\|_F)$ and write $\{\mA_{j,k}\}_{k=1}^{N_2'}$ for the $\varepsilon'$-covering set of $\mF(\mB_{2\delta}(\{\mA_j\}) \mid \mA_j, X_{1:n})$.

We show that $\{\mA_{j,k}\}_{j,k=1}^{N'_1, N'_2}$ is a covering set of $\mF(\mfB_{K,\delta} \mid X_{1:n})$.
Fix $\mA \in \mB_\delta(\mu)$, and pick $\mA_j$ from the covering set $\{\mA_j\}_{j=1}^{N_1'}$ such that $j = \argmin_{j' =1,...,N_1'} \|\mA_{j'} - \Pi_{\mu}(\mA)\|$.
From the proof of Lemma {\ref{lem:covering_neighbour}}, we know that $\mA \in \mB_{2\delta}(\{\mu_j\}).$
Hence, we obtain
\begin{align*}
    &\log \mN(\varepsilon, \mF(\mB_\delta(\mu) \mid X_{1:n}), \|\cdot\|_F)\\
    &\leq \log \mN(\varepsilon, \mu, \|\cdot\|_F) \\
    & \quad + \sup_{\mA \in \mu} \log \mN(\varepsilon, \mF(\mB_{2\delta}(\mu) \mid X_{1:n}), \|\cdot\|_F).
\end{align*}
Combining this with the inequality in \eqref{ineq:bound_f_covering} yields the statement.
\end{proof}

\section{Proof of Generalization Bound (Theorem \ref{thm:bound_gen})}

This proof combines the lower bound on the stagnation probability (Theorem \ref{thm:stay_prob_regularized}), the entropy bound (Proposition \ref{prop:covering}), and the standard empirical process technique based on the Rademacher complexity (summarised in \cite{gine2016mathematical}, for example).
For the proof, we define the empirical Rademacher complexity of the function class $\mF$ as
\begin{align*}
    \fR_n(\mF) := \Ep_{u_{1:n} | X_{1:n}}\left[\sup_{f \in \mF} \frac{1}{n} \sum_{i=1}^n u_i f(X_i) \right],
\end{align*}
where $u_i,i=1,...,n$ are independent Rademacher variables taking values $1$ or $-1$ with probability $1/2.$
\begin{proof}[Proof of Theorem \ref{thm:bound_gen}]
By Theorem \ref{thm:stay_prob_regularized}, $\mA_t \in \mfB_{K,\delta}$ holds for all $t \geq \bar{t}$ with some $\bar{t}$ and probability at least $p^*$. We now analyze the output of the SGD algorithm $\mA_T$ for $T \geq \bar{t}$ and can thus assume $\mA_T \in \mfB_{K,\delta}$.

We bound the generalization gap as usual. With probability $1-\varepsilon$, 
\begin{align*}
    &R( \mA_T) - R_n(\mA_T)\\
    & \leq \sup_{\mA \in \mfB_{K,\delta}} R(\mA) - R_n(\mA) \\
    &\leq 2 \fR_n( \{\ell (\cdot, f(\cdot, \mA)) \mid \mA \in \mfB_{K,\delta}\} ) + 3B\sqrt{ \frac{\log (1/\varepsilon)}{2n}} \\
    & \leq 2 \fR_n( \mF(\mfB_{K,\delta}) ) + 3B\sqrt{ \frac{\log (1/\varepsilon)}{2n}}\\
    & = 2 \fR_n( \mF(\mfB_{K,\delta} \mid X_{1:n}) ) + 3B\sqrt{ \frac{\log (1/\varepsilon)}{2n}},
\end{align*}
where the last inequality follows from the contraction inequality (Theorem 3.2.1 in \cite{gine2016mathematical}) associated with the Lipschitz continuity of $\ell$, which is implied by the differentiability setting of $\ell$ described in Section \ref{sec:setting_supervised_learning}. 

The empirical Rademacher complexity is bounded by the metric entropy bound of $\mfB_{K,\delta}$. 
The standard argument based on Dudley's integral (e.g. Theorem 5.22 in \cite{wainwright2019high} and Lemma A.5 in \cite{bartlett2017spectrally}) and Proposition \ref{prop:covering} yield
\begin{align*}
    &\fR_n( \mF(\mB_{\delta} (\mu) \mid X_{1:n}) ) \\
    & \lesssim \underbrace{ \frac{1}{\sqrt{n}} \int_0^{B} \sqrt{C_\mu d \log\Big(\frac K{\varepsilon'}\Big)} \, d \varepsilon' }_{=:T_1} \\
    &\quad + \underbrace{ \inf_{\alpha > 0} \left\{ \frac{4\alpha}{ \sqrt{n}}  + \frac{12}{n} \int_\alpha^{\sqrt{n}} \sqrt{\frac{ n h_0 S^2 L^2 \delta^2 \log (2 \Bar{h}^2) }{(\varepsilon')^2}} \, d \varepsilon' \right\}}_{=: T_2}.
\end{align*}
To bound the first term $T_1,$ integration by parts and some calculations give
\begin{align*}
    T_1 
    &= \sqrt{\frac{C_\mu d}{n}} \int_0^B \left( B \sqrt{\log K} + \int_{\infty}^{\sqrt{\log(1/B)}} y \exp(-y^2) (-2y) dy \right) \\
    &\leq \sqrt{\frac{C_\mu d}{n}} (B \sqrt{\log K} + \sqrt{B} + \sqrt{\pi}). 
\end{align*}
For the second term $T_2$, we follow the calculation in Lemma A.8 in \cite{bartlett2017spectrally} and obtain
\begin{align}
    T_2 &= \inf_{\alpha > 0} \left\{ \frac{4\alpha}{ \sqrt{n}}  + \frac{ 12 \sqrt{ h_0 S^2 L^2 \delta^2 \log (2 \Bar{h}^2) }}{\sqrt{n}} \log (\sqrt{n}/\alpha) \right\}  \notag \\
    &\lesssim \frac{SL\delta \log^{1/2}\Bar{h}\log n}{\sqrt{n}}. \label{ineq:bound_constant_mu0}
\end{align}
Finally, we have
\begin{align*}
    &\fR_n\big(\mF(\mfB_{K,\delta} \mid X_{1:n}) \big) \leq C_{\mu,h_0,B} \frac{\sqrt{d \log K} + SL \delta \log^{1/2}\Bar{h} \log n}{\sqrt{n}},
\end{align*}
proving the assertion. 
\end{proof}

\section{Proof of Optimization Bound (Theorem \ref{thm:bound_gen_var})}

The following lemma establishes a connection between local essential convexity and gradients.

\begin{theorem}[Generalization Gap Bound in Neighbourhood] \label{thm:gen1}
    Suppose Assumption \ref{asmp:essential_convex},  \ref{asmp:local_smoothness}, and \ref{asmp:sample_gradient_noise} hold.
    If, for a positive constant $c$, the learning rate is $\eta_t = c ( C_{\nabla R} L \Bar{h}^2)^{-1} / t$ for all $t \geq \Bar{t},$ with some $c > 0$, then, conditionally on the event $\{\mA_t \in \mB_\delta(\mu): t \in [\Bar{t}, T]\}$,
    we obtain 
    \begin{align*}
        \Ep_{\mE}[R(\mA_T) - R(\Pi_\mu(\mA_T))] \leq \frac{C_{\Bar{t}, \xi, \mu, \nabla R} \delta^2 (1 + m^{-1} + n^{-1/2})}{T}.
    \end{align*}
\end{theorem}
\begin{proof} For $\mA_t \in \mB_\delta(\mu),  t=1,..,T$, we obtain using \eqref{def:reguralized_SGD}, 
\begin{align*}
    &\|\mA_{t+1} - \Pi_\mu(\mA_{t+1})\|_F^2 \\
    & \leq \|\mA_{t+1} - \Pi_\mu(\mA_{t})\|_F^2  \\
    &= \|\mA_{t} - \eta_t \nabla {R}_n(\mA_t) + m^{-1/2} U_{t+1}(\mA_t) -  \Pi_\mu(\mA_{t})\|_F^2 \\
    &= \|\mA_{t}  -  \Pi_\mu(\mA_{t})\|_F^2 + 2\eta_t^2 \| \nabla {R}(\mA_t) + m^{-1/2} U_{t+1}(\mA_t)\|_F^2 \\
    & \quad + 2\eta_t^2 \|\nabla R_n(\mA_t) - \nabla R(\mA_t)\|_F^2  \\
    & \quad - 2 \eta_t \langle \mA_{t}  -  \Pi_\mu(\mA_{t}), \nabla {R}(\mA_t)\rangle  \\
    &\quad - 2 \eta_t \langle \mA_{t}  -  \Pi_\mu(\mA_{t}),  m^{-1/2} U_{t+1}(\mA_t) \rangle.
\end{align*}
By Lemma \ref{lem:smooth_convex_lower} with Assumption \ref{asmp:local_smoothness}, we have 
\begin{align*}
    &\langle \nabla R(\mA_t), \mA - \Pi_{\mu} (\mA_t) \rangle \\
    &\leq R(\mA_t) - R(\Pi(\mA_t)) + \frac 12 C_{\nabla R} L^{1/2} \Bar{h} \|\mA_t - \Pi_\mu(\mA_t)\|_F^2. 
\end{align*}
This inequality and Assumption \ref{asmp:local_smoothness} allow us to bound $\|\mA_{t} - \Pi_\mu(\mA_{t})\|_F^2$ further by 
\begin{align}
    &\|\mA_{t+1} - \Pi_\mu(\mA_{t+1})\|_F^2 \notag \\ 
    &\leq \|\mA_{t}  -  \Pi_\mu(\mA_{t})\|_F^2 + 2\eta_t^2 \| \nabla {R}(\mA_t) + m^{-1/2} U_{t+1}(\mA_t)\|_F^2  \notag \\
    &\quad + 2 \eta_t^2 \Xi_{n,\varepsilon}^2 - 2\eta_t L^{1/2} \Bar{h} C_{\nabla R}  \|\mA_{t} - \Pi_\mu(\mA_{t})\|_F^2 \notag  \\
    & \quad - 2 \eta_t \langle \mA_{t}  -  \Pi_\mu(\mA_{t}),  m^{-1/2} U_{t+1}(\mA_t) \rangle  \notag \\
    &= (1 -  3\eta_t L^{1/2} \Bar{h} C_{\nabla R}) \|\mA_{t}  -  \Pi_\mu(\mA_{t})\|_F^2 \notag \\
    & \quad  + \eta_t^2 \| \nabla {R}(\mA_t)\|_F^2 + \eta_t^2 m^{-1}\| U_{t+1}(\mA_t)\|_F^2 2 \eta_t^2 \Xi_{n,\varepsilon}^2 \notag  \\
    & \quad  + 2\eta_t m^{-1/2} \langle\nabla {R}(\mA_t), U_{t+1}(\mA_t)\rangle \notag  \\
    & \quad - 2 \eta_t \langle \mA_{t}  -  \Pi_\mu(\mA_{t}),  m^{-1/2} U_{t+1}(\mA_t) \rangle. \notag
\end{align}
Taking expectation, applying Cauchy-Schwartz inequality, and using the inequality \eqref{ineq:bound_nabla_R}, 
we find
\begin{align}
    &\Ep_\mE [\|\mA_{t+1} - \Pi_\mu(\mA_{t+1})\|_F^2] \notag \\
    &\leq (1 -  3\eta_t L^{1/2} \Bar{h} C_{\nabla R}) \Ep_\mE[\|\mA_{t}  -  \Pi_\mu(\mA_{t})\|_F^2] \notag \\
    & \quad + \eta_t^2 C_{\nabla R}^2 L \Bar{h}^2\delta^2 + \eta_t^2 m^{-1} \Ep_\mE [\| U_{t+1}(\mA_t)\|_F^2] + \eta_t^2 \Xi_{n}^2, \label{ineq:opt_bound}
\end{align}
where $\Xi_{n} = \int_0^\infty  \Xi_{n,\varepsilon} \mone\{\varepsilon \geq t\}dt$.
Further, we bound the last term by
\begin{align*}
    \sup_{\mA\in \mB_\delta(\mu)}\Ep[ \| U_{t+1}(\mA) \|_{F}^2] &\leq \sup_{\mA\in \mB_\delta(\mu)} \mathrm{Trace}(G(\mA)) \\
    &\leq D c_G  \\
    &\leq L \Bar{h}^2 c_G.
\end{align*}
Here, we recall that $G(\mA)$ is the covariance of $U_{t+1}(\mA)$, $c_G$ is the upper bound of its largest eigenvalues, and $D$ is the total number of network parameters. In particular, $D \leq L \Bar{h}^2$ holds.
Together with \eqref{ineq:opt_bound}, we find
\begin{align*}
    &\Ep_\mE[\|\mA_{t+1} - \Pi_\mu(\mA_{t+1})\|_F^2]\\
    &\leq (1 -  3\eta_t L^{1/2} \Bar{h} C_{\nabla R}) \Ep_\mE[\|\mA_{t}  -  \Pi_\mu(\mA_{t})\|_F^2] \\
    & \quad + \eta_t^2 C_{\nabla R}^2 L \Bar{h}^2\delta^2 + \eta_t^2 m^{-1} L \Bar{h}^2 c_G+ 2 \eta_t^2 \Xi_{n}^2.
\end{align*}
Setting $\eta_t = (2t+1)/(2 L^{1/2} \bar{h} C_{\nabla R} (t+1)^2)$ and substituting this expression in the first term of the upper bound in the previous display gives
\begin{align*}
    &\Ep_\mE[\|\mA_{t+1} - \Pi_\mu(\mA_{t+1})\|_F^2] \\
    &\leq \frac{t^2}{(t+1)^2} \Ep_\mE[\|\mA_{t}  -  \Pi_\mu(\mA_{t})\|_F^2] \\
    & \quad + \eta_t^2 C_{\nabla R}^2 L \Bar{h}^2\delta^2 + \eta_t^2 m^{-1} L \Bar{h}^2 c_G+ 2 \eta_t^2 \Xi_{n}^2.
\end{align*}
Multiply with $(t+1)^2$ on both sides and putting $\Delta_t := t^2 \Ep_\mE[\|\mA_{t}  -  \Pi_\mu(\mA_{t})\|_F^2]$, we obtain
\begin{align*}
    \Delta_{t+1} \leq \Delta_t + (t+1)^2 \eta_t^2  (C_{\nabla R}^2 L \Bar{h}^2\delta^2 + m^{-1} L \Bar{h}^2 c_G + 2  \Xi_{n}^2).
\end{align*}
Applying this formula iteratively yields
\begin{align*}
    \Delta_T \leq \Delta_{\Bar{t}} + \sum_{t = \Bar{t}}^{T-1} (t+1)^2 \eta_t^2  (C_{\nabla R}^2 L \Bar{h}^2\delta^2 + m^{-1} L \Bar{h}^2 c_G + 2  \Xi_{n}^2).
\end{align*}
Substituting $\eta_t = (2t+1)/(2 L \bar{h}^2 C_{\nabla R} (t+1)^2)$ implies that
\begin{align*}
    &\Ep_\mE[\|\mA_T - \Pi_\mu(\mA_{T})\|_F^2] \\
    &\leq \frac{(\Bar{t}+1)^2}{(T+1)^2}\Ep_\mE[\|\mA_{\Bar{t}} - \Pi_\mu(\mA_{\Bar{t}})\|_F^2] \\
    & \quad + \frac{1}{(T+1)^2} \sum_{t = \Bar{t}}^{T-1} (t+1)^2 \eta_t^2  (C_{\nabla R}^2 L \Bar{h}^2\delta^2 + m^{-1} L \Bar{h}^2 c_G + 2  \Xi_{n}^2)\\
    &\leq \frac{(\Bar{t}+1)^2}{(T+1)^2} \delta^2 \\
    & \quad + \frac{C_{\nabla R}^2 L \Bar{h}^2\delta^2 + m^{-1} L \Bar{h}^2 c_G + 2C_{\xi, \mu} n^{-1/2} L^2 \delta^2 \Bar{h}^4}{(T+1)^2  } \\
    & \qquad \times \sum_{t=\Bar{t}}^{T-1} \frac{1}{L^2 \Bar{h}^4 C_{\nabla R}^2} \cdot \frac{(2t+1)^2}{(t+1)^2} \\
    &\leq \frac{C_{\Bar{t}}}{(T+1)^2} \delta^2 \\
    & \quad  + 4\frac{C_{\nabla R}^2 L \Bar{h}^2\delta^2 + m^{-1} L \Bar{h}^2 c_G + 2C_{\xi, \mu} n^{-1/2} L^2 \delta^2 \Bar{h}^4}{(T+1)^2} \cdot \frac{T - \Bar{t} - 1}{L^2 \Bar{h}^4 C_{\nabla R}^2} \\
    &\leq \frac{C_{\Bar{t}, \xi, \mu, \nabla R} \delta^2 (1 + m^{-1} + n^{-1/2})}{T}.
\end{align*}
The second inequality follows from Lemma \ref{lem:norm_F_L21} and the fact $\mA_{\Bar{t}} \in \mB_\delta(\mu)$; the third inequality follows from $\frac{(2t+1)^2}{(t+2)^2} \leq \frac{(2t+2)^2}{(t+2)^2} \leq 2^2$.
By the Lipschitz continuity of $R(\mA)$, we obtain the statement.
\end{proof}

\begin{lemma} \label{lem:gradient_convergence_entropy}
    Suppose Assumption \ref{asmp:local_smoothness} and \ref{asmp:sample_gradient_noise} hold. For any $\delta'> 0$, we have with probability $1-\delta'$,
    \begin{align*}
        \sup_{\mA \in \mB_\delta(\mu)}\| \nabla R_n(\mA) - \nabla R(\mA)\|_{F}\leq  \frac{C_{\xi,\mu}L \delta \Bar{h}^2}{n^{1/4} ({\delta'})^{1/2}} \sqrt{ \log\Big(\frac{\Bar{h}D}{\delta'}\Big)}.
    \end{align*}
\end{lemma}
\begin{proof} 
The proof combines the proof of Theorem 1 in \cite{mei2018landscape} with the entropy bound in Lemma \ref{lem:covering_neighbour_point}.
Let $N=N(\varepsilon) = \mN(\varepsilon, \mB_\delta(\mu), \|\cdot\|_{F})$ and write $\{\mA_1,...,\mA_N\} \subset \mB_\delta(\mu)$ for the centers of the covering.
Set $j(\mA) := \argmin_{j = 1,...,N} \| \mA - \mA_j\|_{F}$. 
Also, let $V_\varepsilon$ be an $\varepsilon$-cover of the unit Euclidean ball in $\R^D$.

Decompose the distance $\| \nabla R_n(\mA) - \nabla R(\mA)\|_F$ into the following three terms via
\begin{align*}
    \| \nabla R_n(\mA) - \nabla R(\mA)\|_{F} & \leq \|\nabla R(\mA) - \nabla R(\mA_{j(\mA)})\|_{F} \\
    & \quad + \| \nabla R(\mA_{j(\mA)}) -  \nabla R_n(\mA_{j(\mA)})\|_{F} \\
    & \quad + \|\nabla R_n(\mA_{j(\mA)}) - \nabla R_n(\mA)\|_{F} \\
    &=: B_1(\mA) + B_2(\mA) + B_3(\mA).
\end{align*}
We now bound suprema over $\mB_\delta(\mu)$ of the terms $B_1(\mA),B_2(\mA)$ and $B_3(\mA)$.

Firstly, applying Lemma \ref{lem:gradient_lip}, we bound $\sup_{\mA \in \mB_\delta(\mu)} B_1(\mA)$ by
\begin{align*}
    &\sup_{\mA \in \mB_\delta(\mu)} B_1(\mA) \\
    &\leq \sup_{\mA \in \mB_\delta(\mu)} \|\mA - \mA_{j(\mA)}\|_F \cdot \sup_{\mA \in \mB_\delta(\mu)} \frac{\| \nabla R(\mA) - \nabla R(\mA_{j(\mA)}) \|_F}{\|\mA - \mA_{j(\mA)}\|_F} \\
    & \leq \varepsilon C_{\nabla R}L^{1/2} \Bar{h}.
\end{align*}
Using this inequality and Markov's inequality, for $\varepsilon_1 > 0$, we obtain
\begin{align*}
    \Pr\left( \sup_{\mA \in \mB_\delta(\mu)} B_1 (\mA) \geq \varepsilon_1 \right) &\leq \varepsilon_1^{-1} \Ep\left[\sup_{\mA \in \mB_\delta(\mu)} B_1 (\mA) \right] \\
    &\leq \frac{\varepsilon C_{\nabla R}L^{1/2} \Bar{h}}{\varepsilon_1}.
\end{align*}
Hence, if $\varepsilon_1 \geq 2 \varepsilon C_{\nabla R}L^{1/2} \Bar{h}/ \delta'$, we obtain $\Pr(B_1 \geq \varepsilon_1) \leq \delta'/2$.

Secondly, we similarly bound the term $\sup_{\mA \in \mB_\delta(\mu)}B_3(\mA)$ by
\begin{align*}
    &\Pr \left(\sup_{\mA \in \mB_\delta(\mu)}B_3(\mA) \geq \varepsilon_1 \right) \\
    &\leq \varepsilon_1^{-1} \Ep \left[\sup_{\mA \in \mB_\delta(\mu)} \frac{\|\nabla R_n(\mA_{j(\mA)}) - \nabla R_n(\mA)\|_{F}}{\|\mA - \mA_{j(\mA)}\|_{F}} \right. \\
    & \qquad \qquad \left. \times \sup_{\mA \in \mB_\delta(\mu)}\|\mA - \mA_{j(\mA)}\|_{F} \right]\\
    &\leq \frac{\varepsilon C_{\nabla R}L^{1/2} \Bar{h}}{\varepsilon_1}.
\end{align*}
Hence $\Pr(\sup_{\mA \in \mB_\delta(\mu)}B_3(\mA) \geq \varepsilon_1) = 0,$ whenever $\varepsilon_1 \geq 2 \varepsilon C_{\nabla R}L^{1/2} \Bar{h}/ \delta'$.

To bound finally $\sup_{\mA \in \mB_\delta(\mu)}B_2(\mA)$, we apply Lemma 3 in \cite{mei2018landscape} and the union bound.
Let $V_{1/2}$ be an $(1/2)$-cover with $|V_{1/2}| \leq 6^D$. By Assumption \ref{asmp:sample_gradient_noise}, the random variable $\rho_{j,i} := (\nabla R(\mA_{j}) - \nabla \ell (Z_i, \mA_{j}))$ is $\xi^2$-sub-Gaussian. It follows that $\frac{1}{n} \sum_{i=1}^n \rho_{j,i} = \frac{1}{n} \sum_{i=1}^n  (\nabla R(\mA_{j}) - \nabla \ell (Z_i, \mA_{j}))$ is $\xi^2/n$-sub-Gaussian. Thus,
\begin{align*}
    &\Pr\left(\sup_{\mA \in \mB_\delta(\mu)}B_2(\mA) \geq \varepsilon_1\right) \\
    &\leq \Pr \left( \sup_{j=1,...,N(\varepsilon)} \ \sup_{v \in V_{1/2}} \ \big\langle v,\nabla R(\mA_{j}) - \nabla R_n(\mA_{j})  \big\rangle \geq \frac{\varepsilon_1}{2} \right) \\
    & \leq N(\varepsilon) D \log 6 \sup_{j = 1,...,N(\varepsilon)} \sup_{v \in V_{1/2}} \Pr \left(  \left\langle v, \frac{1}{n} \sum_{i=1}^n  \rho_{j,i} \right\rangle \geq \frac{\varepsilon_1}{2} \right) \\
    & \leq \exp\big(\log (N(\varepsilon) +D \log 6)\big) \exp \left(- \frac{n \varepsilon_1^2 }{16  \xi^2 }\right).
\end{align*}
Choosing
\begin{align*}
    \varepsilon_1^2 \geq \frac{16 \xi^2 
    (\log (N(\varepsilon) + D \log 6) + \log(2/\delta'))}{n},
\end{align*}
implies $\Pr(B_2  \geq \varepsilon_1) \leq \delta'/2$.

By combining all the bounds on $\sup_\mA B_1(\mA)$, $\sup_\mA B_2(\mA)$ and $\sup_\mA B_3(\mA)$, we finally obtain that for
\begin{align*}
    a &\geq \max \left\{ \frac{3 \varepsilon C_{\nabla R}L^{1/2} \Bar{h}}{ \delta'} , \sqrt{\frac{144 \xi^2 
    (\log (N(\varepsilon) + D \log 6) + \log(2/\delta'))}{n}}\right\} \\
    &=: \underline{a},
\end{align*}
we have
\begin{align*}
    &\Pr \left(\sup_{\mA \in \mB_\delta(\mu)}\| \nabla R_n(\mA) - \nabla R(\mA)\|_{F} \geq a \right)\\
    & \leq \Pr \left(\sup_{\mA \in \mB_\delta(\mu)} B_1(\mA) \geq \frac{a}{3} \right) + \Pr \left(\sup_{\mA \in \mB_\delta(\mu)} B_2(\mA) \geq \frac{a}{3} \right) \\
    & \quad + \Pr \left(\sup_{\mA \in \mB_\delta(\mu)} B_3(\mA) \geq \frac{a}{3} \right) \\
    & \leq \frac{\delta'}{2} + 0 + \frac{\delta'}{2} = \delta'.
\end{align*}
Applying Lemma \ref{lem:covering_neighbour} with $K=1$ and setting $\varepsilon = (\delta')^{1/2}/(n^{1/4} 3C_{\nabla R} L^{1/2} \Bar{h})$, we can bound the threshold $\underline{a}$ by
\begin{align*}
    \underline{a} &\leq  \max \left\{ \frac{3 \varepsilon C_{\nabla R}L^{1/2} \Bar{h}}{ \delta'} , \right. \\
    & \left. \qquad \sqrt{\frac{144 \xi^2 
    (4 L \delta^2 \Bar{h}^2 \varepsilon^{-2} \log (2\Bar{h}) + \log (D \log 6) + \log(2/\delta'))}{n}}\right\} \\
    & \leq C_{\xi, \nabla R}\frac{ L \delta \Bar{h}^2 \log^{1/2}(2 \Bar{h}) (\delta')^{-1/2} + \log D + \log (2/\delta')}{n^{1/4}}.
\end{align*}
The assertion follows.
\end{proof}

\section{Additional Results}

\begin{proof}[Proof of Lemma \ref{lem:relu}]
Suppose $X$ follows a uniform distribution on $[-1,1]$ and $Y=1,$ almost surely. Then, $P_{X,Y}$ is a product measure.
Denote the scalar parameters of each layer by $a_1 = A_1$ and $a_2 = A_2$.
Now, we have
\begin{align*}
    R(\mA) &= \int_{-1}^1 (1 - a_2 \sigma(a_1 x))^2 dx \\
    &= 1 - 2 a_2 \int_{-1}^1 \sigma(a_1 x)dx + a_2^2 \int_{-1}^1 \sigma(a_1 x)^2dx \\
    &= 1 - 2 a_2[\max\{0,x\}^2/(2a_1)]_{-a_1}^{a_1} \\
    & \quad + a_2^2[\max\{0, x \}^3/(3a_1)]_{-a_1}^{a_1}\\
    &= 1 - a_1 a_2 + a_1^2 a_2^2 / 3.
\end{align*}
Observe that
\begin{align*}
    \nabla_{1,1,1} R(\mA) = -a_2 + \tfrac{2}{3} a_1 a_2^2,
\end{align*}
and
\begin{align*}
    \nabla_{2,1,1} R(\mA) =- a_1 + \tfrac{2}{3} a_1^2 a_2.
\end{align*}
We only consider the case $ \nabla_{1,1,1} R(\mA)$, because swapping $a_1$ and $a_2$ in $ \nabla_{1,1,1} R(\mA)$ will result in $ \nabla_{2,1,1} R(\mA)$. For $\mA = (a_1,a_2)$ and $\mA' = (a_1',a_2')$, we have
\begin{align*}
    &|\nabla_{1,1,1} R(\mA) - \nabla_{1,1,1} R(\mA')| \\
    &\leq |a_2 - a_2'| + \tfrac{2}{3}(|a_1-a_1'| |a_2^2| + |a_1'|(|a_2||a_2-a_2'| + |a_2'||a_2 - a_2'|)) \\
    & \leq |a_2 - a_2'| + \tfrac{2}{3}B^2 |a_1-a_1'| + \tfrac{4}{3}B^2 |a_2 - a_2'|\\
    &= \left\{(\tfrac{2}{3}B^2)^2|a_1 - a_1'|^2 + (1 + \tfrac{4}{3}B^2)^2|a_2 - a_2'|^2 \right. \\
    & \qquad \left.+ 2(\tfrac{2}{3}B^2)(1 + \tfrac{4}{3}B^2)|a_1 - a_1'||a_2 - a_2'| \right\}^2 \\
    & \leq \left\{\{(\tfrac{2}{3}B^2)^2 + (\tfrac{2}{3}B^2)(1 + \tfrac{4}{3}B^2) \}|a_1 - a_1'|^2 \right. \\
    & \qquad \left. + \{(1 + \tfrac{4}{3}B^2)^2 + (\tfrac{2}{3}B^2)(1 + \tfrac{4}{3}B^2)\}|a_2 - a_2'|^2 \right\}^2\\
    & \leq \left\{(1 + \tfrac{4}{3}B^2)^2 + (\tfrac{2}{3}B^2)(1 + \tfrac{4}{3}B^2)\right\}^2 \left\{|a_1 - a_1'|^2 + |a_2 - a_2'|^2\right\}^2 \\
    &= \left\{(1 + \tfrac{4}{3}B^2)^2 + (\tfrac{2}{3}B^2)(1 + \tfrac{4}{3}B^2)\right\}^2 \|\mA - \mA'\|_F.
\end{align*}
Since $\sqrt{(1 + \tfrac{4}{3}B^2)^2 + (\tfrac{2}{3}B^2)(1 + \tfrac{4}{3}B^2)} \leq \sqrt{2} \sqrt{(1 + \tfrac{4}{3}B^2)^2} = \sqrt{2}(1 + \tfrac{4}{3}B^2),$ the proof is complete.
\end{proof}

\begin{lemma} \label{lem:norm_F_L21}
    For any $\mA \in \R^D$, we have
    \begin{align*}
        \|\mA\|_F \leq \|\mA\|_{L,2,1}, \mbox{~and~}\|\mA\|_{L,2,1} \leq \sqrt{L \Bar{h}} \|\mA\|_F.
    \end{align*}
\end{lemma}
\begin{proof} By Cauchy-Schwartz inequality,
\begin{align*}
    \|\mA\|_F^2 &= \sum_{\ell=1}^{L} \sum_{j=1}^{h_{\ell-1}} \|A_{:,j}\|_2^2 \\
    &\leq \left(\sum_{\ell=1}^{L} \sum_{j=1}^{h_{\ell-1}} \|A_{:,j}\|_2\right)^2 \\
    &= \|\mA\|_{L,2,1}^2 \\
    &\leq L \Bar{h} \sum_{\ell=1}^{L}\sum_{j=1}^{h_{\ell-1}} \|A_{:,j}\|_2^2 \\
    &= \|\mA\|_F^2.
\end{align*}
\end{proof}

\begin{lemma} \label{lem:smooth_convex_lower}
Suppose Assumption \ref{asmp:local_smoothness} holds.
Then, for any $\mA, \mA' \in \mB_\delta(\mu)$ such that $\Pi_\mu(\mA) = \Pi_\mu(\mA')$, 
\begin{align*}
    R(\mA') \leq R(\mA) + \langle \mA' - \mA, \nabla R(\mA) \rangle + \frac 12\sqrt{L}\Bar{h} C_{\nabla R} \|\mA - \mA'\|_F^2.
\end{align*}
\end{lemma}
\begin{proof} Fix $\mA, \mA' \in \mB_\delta(\mu)$ such that $\Pi_\mu(\mA) = \Pi_\mu(\mA')$.
We define a surrogate function $g(t)$ with $t \in \R$ as $g(t) = R(\mA + t(\mA' - \mA))$.
By Cauchy-Schwartz inequality and Assumption \ref{asmp:local_smoothness}, we have
\begin{align*}
    g(1) &= g(0) + \int_0^1 \nabla g(t) dt\\
    &= g(0) + \int_0^1 \langle \mA' - \mA, \nabla  R(\mA + t(\mA' - \mA)) \rangle dt\\
    &=g(0) + \langle \mA' - \mA, \nabla R(\mA) \rangle \\
    & \quad + \int_0^1 \langle \mA' - \mA, \nabla  R(\mA + t(\mA' - \mA)) - \nabla R(\mA) \rangle dt \\
    & \leq g(0) + \langle \mA' - \mA, \nabla R(\mA) \rangle + \int_0^1 t \sqrt{L}\Bar{h} C_{\nabla R} \|\mA - \mA'\|_F^2 dt \\
    & = g(0) + \langle \mA' - \mA, \nabla R(\mA) \rangle + \frac 12 \sqrt{L}\Bar{h} C_{\nabla R} \|\mA - \mA'\|_F^2,
\end{align*}
proving the statement.
\end{proof}

\begin{lemma}\label{lem:gradient_lip}
    Suppose Assumption \ref{asmp:local_smoothness} holds.
    Then, with fixed  $\mu \in \{\mu_1,...,\mu_K\}$, for any $\mA, \mA' \in \mB_{\delta}(\mu)$, we have
    \begin{align*}
        \|\nabla R(\mA) - \nabla R(\mA') \|_F  \leq C_{\nabla R} \sqrt{L} \Bar{h} \|\mA -\mA' \|_F.
    \end{align*}
\end{lemma}
\begin{proof} By Assumption \ref{asmp:local_smoothness},
\begin{align*}
    &\|\nabla R(\mA) - \nabla R(\mA') \|_F \\
    &= \left(\sum_{\ell=1}^L \sum_{j=1}^{h_\ell} \sum_{k=1}^{h_{\ell - 1}}|\nabla_{\ell,j,k} R(\mA) - \nabla_{\ell,j,k} R (\mA')|^2\right)^{1/2} \\
    & \leq \left(\sum_{\ell=1}^L \sum_{j=1}^{h_\ell} \sum_{k=1}^{h_{\ell - 1}} C_{\nabla R}^2 \|\mA -\mA' \|_F^2\right)^{1/2}\\
    &\leq C_{\nabla R} \sqrt{L} \Bar{h} \|\mA -\mA'\|_F.
\end{align*}
\end{proof}

\begin{proof}[Proof of Proposition \ref{prop:gap_example}]
By Lemma D.3 in \cite{zhong2017recovery}, the submatrix of the Hesse matrix $\nabla^2_{A_1} R(\mA^*)$ is strictly positive definite, hence the rank of the Hesse matrix is not less than $h^2$. Since the dimension $\mu$ is bounded by a  codimension of the linear space spanned by the eigenvectors of the Hesse matrix, we can bound the dimension of $\mu$ by $d \leq D - h^2 = h$.

To bound the quantity $S$, consider $\mA =(A_1,A_2) \in \mB_\delta(\mu).$ Due to $\|A_1^*\|_s = \|A_2^*\|_s = 1,$ we have for $j \in \{1,2\},$
\begin{align*}
    \|A_j\|_s &\leq \|A_j - A_j^*\|_s  + \|A_j^*\|_s \leq \|A_j - A_j^*\|_{2,1} + 1.
\end{align*}
Note that $\|A_2\|_{2,1} = \|A_2\|_2$ by regarding the vector $A_2$ as an $h \times 1$ matrix.
Hence we have
\begin{align*}
    S &\leq \sup_{A_1, A_2 : \|A_1 -A_1^*\|_{2,1} + \|A_2- A_2^*\|_{2,1} \leq \delta} (1 + \|A_1\|_{2,1}) (1 + \|A_2\|_{2,1})\\
    &= (1 + \delta/2)^2 \lesssim 1 + \delta^2.
\end{align*}
Finally, we study the constant $C_{C, \mu_0, B}$ in Theorem \ref{thm:bound_gen} substituting $h_0 = h$.
As shown in display \eqref{ineq:bound_constant_mu0}, this constant only affects the second term in the nominator in the bound of Theorem \ref{thm:bound_gen}, proving the statement.
\end{proof}

\section*{Acknowledgment}

We would like to thank the anonymous referees, an Associate Editor and the Editor for their constructive comments that improved the quality of this paper.

\bibliographystyle{ieeetr}
\bibliography{ref}

\end{document}